\documentclass{article}

     \PassOptionsToPackage{numbers, compress}{natbib}

\usepackage[preprint]{neurips_2022}

\usepackage[utf8]{inputenc} %
\usepackage[T1]{fontenc}    %
\usepackage{hyperref}       %
\usepackage{url}            %
\usepackage{booktabs}       %
\usepackage{amsfonts}       %
\usepackage{nicefrac}       %
\usepackage{microtype}      %
\usepackage{xcolor}         %
\usepackage{subfigure}

\usepackage{amsfonts}
\usepackage{amsmath}
\usepackage{amsthm}
\usepackage{amssymb}
\usepackage{dsfont}

\usepackage{xcolor}
\usepackage{color}
\usepackage{graphicx}

\usepackage{verbatim}
\usepackage{xspace} %

\usepackage{enumerate}
\usepackage{enumitem}

\providecommand{\lin}[1]{\ensuremath{\left\langle #1 \right\rangle}}

\providecommand{\norm}[1]{\left\lVert#1\right\rVert}

  \providecommand{\R}{\mathbb{R}} %

  \DeclareMathOperator{\E}{{\mathbb E}}
  \providecommand{\EE}[2]{{\mathbb E}_{#1}\left.#2\right. }  %

  \providecommand{\PP}[2]{{\rm Pr}_{#1}\left[#2\right] }

  \renewcommand{\aa}{\mathbf{a}}
  \providecommand{\bb}{\mathbf{b}}

  \providecommand{\xx}{\mathbf{x}}
  \providecommand{\yy}{\mathbf{y}}

  \providecommand{\cA}{\mathcal{A}}
  
  \providecommand{\cC}{\mathcal{C}}
  \providecommand{\cD}{\mathcal{D}}

  \providecommand{\cN}{\mathcal{N}}
  \providecommand{\cO}{\mathcal{O}}

  \providecommand{\cW}{\mathcal{W}}

  \usepackage{bm}

\usepackage{url}

\renewcommand{\epsilon}{\varepsilon}

\usepackage{thmtools}
\usepackage{thm-restate}

\usepackage[font=small]{caption}

\usepackage{multibib}
\newcites{appendix}{Additional References}%

\usepackage{algorithm}
\usepackage{algorithmic}

\newtheorem{lemma}{Lemma}

\newtheorem{corollary}[lemma]{Corollary}

\newtheorem{definition}[lemma]{Definition}

\newtheorem{assumption}{Assumption}

\newtheorem{example}[lemma]{Example}

\usepackage{pifont}%

\usepackage[compact]{titlesec} 
\titlespacing*{\section}{14pt}{7pt}{4pt}
\titlespacing*{\subsection}{6pt}{3pt}{1pt}
\graphicspath{{figures/}}

\title{Sharper Convergence Guarantees for Asynchronous SGD for Distributed and Federated Learning}

\author{%
  Anastasia Koloskova\\
  EPFL\\
  \texttt{anastasia.koloskova@epfl.ch} \\
  \And
  Sebastian U. Stich \\
   CISPA \\
  \texttt{stich@cispa.de} \\
  \And
  Martin Jaggi \\
  EPFL \\
  \texttt{martin.jaggi@epfl.ch} \\
}

\begin{document}

\maketitle

\begin{abstract}

We study the asynchronous stochastic gradient descent algorithm for distributed training over $n$ workers which have varying computation and communication frequency over time. In this algorithm, workers compute stochastic gradients in parallel at their own pace and return those to the server without any synchronization.
Existing convergence rates of this algorithm for non-convex smooth objectives depend on the maximum gradient delay $\tau_{\max}$ and show that an $\epsilon$-stationary point is reached after $\cO\!\left(\sigma^2\epsilon^{-2}+ \tau_{\max}\epsilon^{-1}\right)$ iterations,  where $\sigma$ denotes the variance of stochastic gradients. \\
In this work (i) we obtain a tighter convergence rate of $\cO\!\left(\sigma^2\epsilon^{-2}+ \sqrt{\tau_{\max}\tau_{avg}}\epsilon^{-1}\right)$ \emph{without any change in the algorithm} %
where~$\tau_{avg}$ is the average delay, which can be significantly smaller than $\tau_{\max}$. 
We also provide (ii) a simple delay-adaptive learning rate scheme, under which asynchronous SGD achieves a convergence rate of $\cO\!\left(\sigma^2\epsilon^{-2}+ \tau_{avg}\epsilon^{-1}\right)$, and does not require any extra hyperparameter tuning nor extra communications. Our result allows to show \emph{for the first time} that asynchronous SGD is \emph{always faster} than mini-batch SGD. 
In addition, (iii) we consider the case of heterogeneous functions motivated by federated learning applications and improve the convergence rate by proving a weaker dependence on the maximum delay compared to prior works. In particular, we show that the heterogeneity term in convergence rate is only affected by the average delay within each worker. 
\end{abstract}

\section{Introduction}
The stochastic gradient descent (SGD) algorithm \cite{Robbins:1951sgd,Bottou2018:book} and its variants (momentum SGD, Adam, etc.) form the foundation of modern machine learning and frequently achieve state of the art results. With recent growth in the size of models and available training data, parallel and distributed versions of SGD are becoming increasingly important~\cite{zinkevich2010parallelized,Dekel2012:minibatch,dean2012large}. Without those, modern state-of-the art language models \cite{shoeybi2019megatron}, generative models~\cite{ramesh2021zero,ramesh2022hierarchical}, and many  others~\cite{wang2020survey} would not be possible. 
In the distributed setting%
, also known as data-parallel training, optimization is distributed over many compute devices working in parallel (e.g. cores, or GPUs on a cluster) in order to speed up training. 
Every worker computes gradients on a subset of the training data, and the resulting gradients are aggregated (averaged) on a server.

The same type of SGD variants also form the core algorithms for federated learning applications~\cite{McMahan16:FedLearning,Kairouz2019:federated} where the training process is naturally distributed over the user devices, or clients, that keep their local data private, and only transfer the (e.g. encrypted or differentially private) gradients to the server.

A rich literature exists on the convergence theory of above mentioned parallel SGD  methods, %
see e.g. \cite{Dekel2012:minibatch,Bottou2018:book} and references therein. Plain parallel SGD still faces many challenges in practice, motivating research on various approaches to improve efficiency of distributed learning and mini-batch SGD. This includes for example communication compression techniques \cite{Alistarh2017:qsgd,Alistarh2018:topk,Stich20:error-feedback,Vogels19:power}, decentralized communication \cite{Lian2017:decentralizedSGD, Assran:2018sdggradpush,Nedic2020:survey,koloskova2020unified} 
or performing several local SGD steps on workers before communicating with the server \cite{Mangasarian1994,mcdonald2010distributed,McMahan16:FedLearning,Stich2018:LocalSGD}. %

These approaches use synchronous communication, where workers in each round are required to wait for the slowest one, before being able to start the next round of computations. In the presence of such straggler nodes or nodes that have different computation speeds, other workers face significant idle times. 
\emph{Asynchronous} variants of SGD are aimed to solve such inefficiencies and use  available workers more effectively. 
In asynchronous SGD, each worker starts the next computation immediately after finishing computing its own gradient, without waiting for any other workers. This is especially important in the presence of straggler nodes. Asynchronous algorithms were studied both in distributed and federated learning settings \cite{Recht2011:hogwild,mania2017:perturbed_iterate,%
Leblond18:async_svrg_saga_etc,Stich20:critical_params,Nguyen22:FedBuff}. In this paper we focus on such challenging asynchronous variants of SGD and provide an improved theoretical analysis of convergence compared to prior works.

Most existing work has studied the convergence behavior of asynchronous SGD for the setting of homogeneous distributed training data, where worker's objectives are i.i.d. . This assumption however is only realistic e.g. in shared-memory implementations where all processes can access the same data~\cite{Recht2011:hogwild}. Under this assumption, it can be proven that asynchronous SGD finds an $\epsilon$-approximate stationary point (squared gradient norm bounded by $\epsilon$) in $\cO\big( \frac{\sigma^2}{\epsilon^2} + \frac{\tau_{\max}}{\epsilon}  \big)$ iterations~\cite{Stich20:error-feedback}, for smooth non-convex functions. This complexity bound  depends on the maximum delay of the gradients $\tau_{\max}$ %
and the gradient variance $\sigma > 0$.
Unfortunately, the maximal delay is a very pessimistic metric, not well reflecting the true behavior in practice. For instance, if a worker struggles just once, the maximum delay is  large, while we would still expect reasonable overall convergence.

Two recent works \cite{Cohen2021_pickySGD, Aviv21:delayed_average} tackle this issue by proposing two new delay-adaptive algorithms that achieve a convergence rate that depends only on the average delay of the applied gradients, with \citet{Aviv21:delayed_average} considering only the convex optimization and \citet{Cohen2021_pickySGD} providing a rate of $\cO\big( \frac{\sigma^2}{\epsilon^2} + \frac{\tau_{avg}}{\epsilon}  \big)$ for smooth non-convex functions. %
The average delay can be much smaller than the maximal delay, and thus these methods are robust to rare stragglers. However, \citet{Cohen2021_pickySGD} requires twice more communications at every step, and an extra hyperparameter to tune. \citet{Aviv21:delayed_average} analyze only convex functions and assume a bound on the  variance of the delays, which can frequently degrade with the maximum delay $\tau_{\max}$. Moreover, those works require the assumption that gradients are uniformly bounded.

In the realistic case of heterogeneous objective functions, that is in particular relevant in federated learning applications~\cite{Kairouz2019:federated}, all the existent convergence rates of asynchronous SGD depend on the maximum delay \cite{Nguyen22:FedBuff}. 

\textbf{Contributions.}
\begin{itemize}[nosep,leftmargin=12pt,itemsep=2pt]
	\item For standard asynchronous SGD with constant stepsize, and with non-convex $L$-smooth homogeneous objective functions, we prove the tighter convergence rate of $\cO\big( \frac{\sigma^2}{\epsilon^2} + \frac{\sqrt{\tau_{avg} \tau_{\max}}}{\epsilon} \big)$ to $\epsilon$-small error. Under the additional assumption of bounded gradients, we obtain a convergence rate of $\cO\big(\frac{\sigma^2}{\epsilon^2} + \frac{\tau_{avg} G}{\epsilon^{3/2}} + \frac{\tau_{avg}}{\epsilon}\big)$ where $G$ is the bound on the norm of  gradients. The previously best known rate was $\cO\big(\frac{\sigma^2}{\epsilon^2} + \frac{\tau_{\max}}{\epsilon} \big)$. 
		
	\item With homogeneous objective functions, we provide a delay-adaptive stepsize scheme that does not require tuning of any extra hyperparameters, and converges at the rate of $\cO\big( \frac{\sigma^2}{\epsilon^2} + \frac{\tau_{avg}}{\epsilon}  \big)$ for non-convex $L$-smooth functions.
	\item This result allows us to show that asynchronous SGD is always better than mini-batch SGD regardless of the delays pattern (under assumption that the server can perform operations with zero time). 

	\item We also consider distributed optimization with heterogeneous objectives where the delays can depend on the nodes and give the convergence rate of $\cO\!\big(\frac{\sigma^2}{\epsilon^2} + \frac{\zeta^2}{\epsilon^2} + \frac{\sqrt{\tau_{avg} \frac{1}{n} \sum_{i = 1}^n \zeta_i^2 \tau^i_{avg}}}{\epsilon^{\frac{3}{2}}}+ \frac{\sqrt{\tau_{avg} \tau_{\max}}}{\epsilon}\big)$, where $\zeta_i$'s measure functions heterogeneity and $\bar \tau_i$ is the average delay of node $i$. This rate improves over the best previously-known results that had worse dependence on the maximum delay~$\tau_{\max}$. 
\end{itemize}

\section{Related Work}
\paragraph{Asynchronous SGD.}
The research field of asynchronous optimization can be traced back at least to~1989 \cite{Bertsekas1989:parallel}. Recent works are heavily focused on its SGD variants, such as Hogwild! SGD \cite{Niu2011:hogwild}
 which deals with coordinate-wise asynchronity. \citet{nguyen18:SGD_and_hogwild_without_bounded_gradient} provided a tighter convergence analysis by removing the bounded gradient assumption. Our work does not focus on such a coordinate-wise asynchrony as it relies on sparsity assumption that is not realistic in modern machine learning applications. 
\citet{mania2017:perturbed_iterate} introduces the perturbed iterate framework which enabled theoretical advances with tighter convergence rates~\cite{Stich20:error-feedback,Stich20:critical_params}. \citet{Leblond18:async_svrg_saga_etc} focus on asynchronous variance-reduced methods.

Many works \cite{Agarwal2011:delayed, Chaturapruek2015:noise, Feyzmahdavian2016:async, arjevani20:delayedSGD, sra16:adadelay, lian2015asynchronous, Stich20:error-feedback, Dutta18:stale_gradients} focused on asynchronous SGD variants where workers communicate with the server without any synchronization, but these communications are considered to be atomic. All of these works provide convergence guarantees that depend on the maximum delay~$\tau_{\max}$ with \cite{arjevani20:delayedSGD, Stich20:error-feedback} providing the first tight convergence rates under assumption that the delays are always constant for quadratic and general (convex, strongly convex and non-convex) functions correspondingly.  \citet{Stich20:critical_params} showed a connection of large batches and delays, although still depending only on the maximum delay. 
\citet{Mathieu21:continuous_time_decentralized_delays} consider a continuized view of the time (rather than classical per-iteration time) for asynchronous algorithms on a decentralized network.

\paragraph{Delay-adaptive methods.} The works \cite{Zheng17:delay_compensation, Zhang16:async_weight_down_lr, sra16:adadelay, Xuyang22:delay-adaptive-stepsizes, McMahan14:delay_adaptive_online, Dutta18:stale_gradients} considered delay-adaptive schemes to mitigate adversarial effect of stragglers, however with convergence rates that still depend on the maximum delay $\tau_{\max}$. Only \citet{Cohen2021_pickySGD} in the non-convex, and \citet{Aviv21:delayed_average} in the convex case were able to obtain convergence rates depending on the average delay $\tau_{avg}$. Concurrent to our work, \citet{Mishchenko22:async} provide a delay-adaptive scheme similar to ours and derive convergence guarantees depending on the concurrency $\tau_C$. %
However, they did not consider asynchronous SGD with constant stepsizes, nor the bounded gradients case. %
Moreover, for heterogeneous functions their method with delay adaptive stepsizes does not converge and only reaches an approximate solution (up to heterogeniety), while in our work we prove convergence for a different method with carefully tuned constant stepsizes.

\paragraph{Asynchronous federated learning.} In typical federated learning (FL) applications \cite{McMahan16:FedLearning}, clients or workers frequently have very different computing powers/speed. This makes especially appealing for practitioners to use asynchronous algorithms for FL \cite{Stich2018:LocalSGD, Nguyen22:FedBuff, Avdiukhin21:FL_delayed, Haibo21:anarchicFL, Xinran21:MIFA, Arda16:distributed_async, Glasgow20:async_variance_reduced_distributed, Yikai20:unavailable_devices_cd} with many of these works focusing on correcting for unequal participation ratio of different clients \cite{Yikai20:unavailable_devices_cd, Glasgow20:async_variance_reduced_distributed, Xinran21:MIFA, Arda16:distributed_async, Haibo21:anarchicFL} by implementing variance reduction techniques on the server. 
\citet{Nguyen22:FedBuff} introduce the FedBuff algorithm that is very close to the algorithm that we consider in this work and show its practical superiority over classical synchronous FL algorithms.

\section{Setup}\label{sec:setup}
We consider optimization problems where the components of the objective function (i.e.\ the data for machine learning problems) is distributed across $n$ nodes (or clients), \vspace{-1mm}
\begin{align}
\min_{\xx \in \R^d} \bigg[f(\xx) := \frac{1}{n}\sum_{i = 1}^{n} \big[f_i(\xx)= \E_{\xi \sim \cD_i} F_i(\xx,\xi)\big]\bigg]\,. \label{eq:problem}
\end{align}
Here $f_i \colon \R^d \to \R$ denotes the local loss function that is accessible to the node  $i$, $i \in [n] := \{1,\dots n\}$. Each $f_i$ is a stochastic function $f_i(\xx) = \E_{\xi \sim \cD_i} F_i(\xx, \xi)$ and clients can only access stochastic gradients $\nabla F_i(\xx, \xi)$. This setting covers deterministic optimization if $F_i(\xx, \xi) = f_i(\xx)$, $\forall \xi$. It also covers \emph{empirical risk minimization} problems by setting $\cD_i$ being a uniform distribution over a local dataset $\{\xi_i^1 \dots \xi_i^{m_i}\}$ of size $m_i$. In this case the local functions $f_i$ can be written as finite sums:\ $f_i(\xx) = \frac{1}{m_i} \sum_{j = 1}^{m_i} F_i(\xx, \xi_i^j)$. 

\paragraph{Assumptions.}
For our convergence analysis we rely on following standard assumptions on the functions $f_i$ and $F_i$:
\begin{assumption}[bounded variance] \label{a:stoch_noise}
	We assume that there exists a constant $\sigma \geq 0$ such that %
	\begin{align}\label{eq:stochastic_noise}
	\E_{\xi\sim \cD_i} \norm{\nabla F_i(\xx, \xi) - \nabla f_i(\xx)} \leq \sigma^2\,, && \forall i \in[n], \forall \xx\in \R^d\,.
	\end{align}
\end{assumption}

\begin{assumption}[bounded function heterogeneity]\label{a:heterogeneity}
	We assume that there exists $n$ constants $\zeta_i \geq 0$, $i \in [n]$  such that 
	\begin{align}\label{eq:bound_heterogeniety}
	\norm{\nabla f_i(\xx) - \nabla f(\xx)}_2^2 \leq \zeta_i^2\,, ~~~~~\forall \xx \in \R^d\,, && \text{and define} ~~~~ \zeta^2 := \textstyle \frac{1}{n} \sum_{i = 1}^n \zeta_i^2\,.
	\end{align}
\end{assumption}

\begin{assumption}[$L$-smoothness]\label{a:lsmooth_nc}
	Each function $f_i \colon \R^d \to \R$, $i \in [n]$
	is differentiable and there exists a constant $L \geq 0$ such that %
	\begin{align}\label{eq:smooth_nc}
	&\norm{\nabla f_i(\yy) - \nabla f_i(\xx) } \leq L \norm{\xx -\yy}\,. & &\forall \xx, \yy \in \R^d\,.
	\end{align}
\end{assumption}
For only \emph{some} of the results we will assume a bound on the gradient norm.
\begin{assumption}[bounded gradient]\label{a:bounded_gradient}
	Each function $f_i \colon \R^d \to \R$, $i \in [n]$
	is differentiable and there exists a constant $G \geq 0$ such that %
	\begin{align}\label{eq:bounded_gradient}
	&\norm{\nabla f_i(\xx)}_2^2 \leq G^2 \,, & &\forall \xx\in \R^d\,.
	\end{align}
\end{assumption}

\section{Homogeneous Distributed Setting}\label{sec:homogeneous}
We start with an important special case of problem \eqref{eq:problem} where the objective functions are identical for all workers, i.e.\ $f_i(\xx)\equiv f_j(\xx)$ for all $i,j \in [n]$, such as in the case of homogeneously (i.i.d.) distributed training data.
Consequently, this implies that Assumption~\ref{a:heterogeneity} holds with $\zeta_i = 0$, $i \in [n]$. Many classical works have focused on asynchronous algorithms under this homogeneous setting (e.g.\ \cite{arjevani20:delayedSGD, Stich20:error-feedback, Agarwal2011:delayed, Feyzmahdavian2016:async, sra16:adadelay, lian2015asynchronous}, see the related work for more references). This setting commonly appears in the datacenter setup for distributed training \cite{dean2012large}, where all nodes (or GPUs) have access to the full dataset or data distribution.  Moreover, this special case allows us to present our main ideas in a simplified way, without complicating the presentation due to heterogeneity. We will later see that most of the results in this section can also be obtained as a corollary of the more general heterogeneous functions case (Section~\ref{sec:heterogeneous}) by setting $\zeta_i = 0~i \in [n]$.

\subsection{Algorithm}
We consider standard asynchronous SGD (also known as delayed SGD, or SGD with stale updates)
as presented in Algorithm~\ref{alg:async-homogeneous-general}, see e.g. \cite{arjevani20:delayedSGD, Stich20:error-feedback, Agarwal2011:delayed, Feyzmahdavian2016:async, sra16:adadelay, lian2015asynchronous}. 
\begin{algorithm}[tb]
	\caption{\textsc{Asynchronous SGD}}\label{alg:async-homogeneous-general}
	\let\oldfor\algorithmicfor
	\begin{algorithmic}[1]
		\INPUT{Initial value $\xx^{(0)} \in \R^d$}\\[1ex]
		\STATE sever selects a set of active workers $\cC_0 \!\subseteq\! [n]$ %
		and sends them $\xx^{(0)}$
		\FOR{$t=0,\dots, T-1$}
		\STATE active workers $\cC_t$ are computing stochastic gradients in parallel at the assigned points
		\STATE once a worker $j_t$ finishes compute, it sends $\nabla F(\xx^{(t - \tau_t)}, \xi_t)$ to the server
		\STATE server updates $\xx^{(t + 1)} = \xx^{(t)} - \eta_t \nabla F(\xx^{(t - \tau_t)}, \xi_t)$ 
		\STATE server selects subset $\cA_t \!\subseteq\! [n]$ %
		 of inactive workers, i.e. $(\cC_t \backslash \{j_t\} )\!\cap\! \cA_t \!=\! \emptyset$, and sends them $\xx^{(t + 1)}\!\!\!\!$
		\STATE update active worker set $\cC_{t + 1} = \cC_t \backslash \{j_t \} \cup \cA_t$
		\ENDFOR
	\end{algorithmic}
\end{algorithm}
First, the server initializes training by selecting an initial active worker set $\cC_0$ and assigning $\xx^{(0)}$ to these workers. Throughout the algorithm, the active workers compute gradients at their own speed, based on their local data. On line 4, once some worker (which we denote as $j_t$) finishes computing its gradient, it sends the result to the server. On line 5 the server incorporates the received---possibly delayed---gradient, using a stepsize $\eta_t$ that can depend on the gradient delay $\tau_t$. The \emph{gradient delay} $\tau_t$ is defined as the difference between the iteration at which worker $j_t$ started to compute the gradient and the iteration $t$ at which it got applied.  We index the stochastic noise of the gradients $\xi_t$ by iteration $t$ to highlight that previous iterates $\xx^{(t^\prime)}$ for $t^\prime \leq t$ do not depend on this stochastic noise. However, the client selects the data sample $\xi_t$ at iteration $t-\tau_t$ when the computation starts. 
After that, on lines 6-7 the server selects the new active workers out of the ones that are currently inactive (including worker $j_t$) and assigns them the latest iterate~$\xx^{(t + 1)}$. 

In contrast to previous works, we explicitly define the set of workers that are busy with computations at every step $t$ as $\cC_t$ (the active workers set). Note that this does not pose any restrictions. %
A main advantage of allowing the sets $\cC_t$ to be different at every step $t$ lies in the possibility to also cover mini-batch SGD as a special case, which we discuss in Example~\ref{ex:minibatch}. %
Our theoretical results depend on the size of these sets $\cC_t$, a.k.a. the \emph{concurrency}. 

\begin{definition}[Concurrency]\label{def:concurrency}
	The \emph{concurrency} $\tau_C^{(t)}$ at step $t$ is defined as the size of the active worker set $\cC_t$, i.e. 
$	\tau_C^{(t)} = |\cC_t|$.
	We also define the maximum and average concurrency as 
	\begin{align*}
	 \tau_C =  \max_{t}\{\tau_C^{(t)} \}\,, && \bar \tau_C = \textstyle \frac{1}{T+ 1} \sum_{t = 0}^T \tau_C^{(t)}\,.
	\end{align*}
\end{definition}

Note that in many practical scenarios, we have a \emph{constant concurrency} of $n$ over time, meaning that all $n$ workers are active at every step, and thus $\tau_C = \bar\tau_C = n$.  %

We discuss two important practical examples that fit into our Algorithm~\ref{alg:async-homogeneous-general}:

\begin{example}[Mini-batch SGD]\label{ex:minibatch} Mini-batch SGD with batch size $n$ can be seen as a special case of Algorithm~\ref{alg:async-homogeneous-general}, as follows: The server (i) in line 1 selects all $n$ workers, $\cC_0 = [n]$; (ii) in line 6 does not select new workers while the gradients from the same batch have not been fully applied yet, i.e. $\cA_t = \emptyset$ if $t \!\!\mod n \neq 0$; (iii) in line 6 selects $\cA_t = [n]$ if $t \!\!\mod n = 0$ to start a new batch. 
\end{example}

\begin{example}[Asynchronous SGD with maximum concurrency] In practical implementations one should always aim to utilize all resources available and thus (i) in line 1 select all available workers $\cC_0 = 0$; (ii) in line 6 select the worker that finished its computations $\cA_t = \{j_t\}$ so that workers are always busy with jobs. 
\end{example}

\subsection{Theoretical analysis: Constant stepsizes}
We first formally define the average and maximum delays. 
\begin{definition}[Average and maximum delays]\label{def:delays}
Let $\{\tau_t\}_{t = 0}^{T - 1}$ be the delays of the applied gradients in Algorithm~\ref{alg:async-homogeneous-general}. We define $\{\tau_i^{\cC_T}\}_{i \in \cC_T \backslash \{ j_T \}}$ as the delays of gradients which are in flight at time $T$, that is they have remained unapplied at the last step. Each $\tau_i^{\cC_T}$ is equal to the difference between the last iteration $T$ and the iteration at which worker $i$ started to compute its last gradient. 
We then define the average and the maximum delays as
\resizebox{\linewidth}{!}{
\begin{minipage}{1.06\linewidth} 
\begin{align}
\tau_{avg} = \frac{1}{T + |\cC_T| - 1}\bigg(\sum_{t = 0}^{T - 1} \tau_t + \sum_{ i \in\cC_T \backslash \{ j_T \} } \tau_i^{\cC_T} \bigg), && \tau_{\max} =\max \left\{\max_{t = 1, \dots T-1} \tau_t ,  \max_{i\in\cC_T \backslash \{ j_T \}} \tau_i^{\cC_T}\right\}.
\end{align}
\end{minipage}}
\end{definition}
We further provide a key observation on the connection between the average delay and the average concurrency. This observation, is one of the essential elements for achieving an improved analysis. 
\begin{restatable}[Key Observation]{remark}{remarkkey}
	\label{rem:conc_avg_delay}
	In Algorithm~\ref{alg:async-homogeneous-general}  the average concurrency $\bar \tau_C$ is connected to the average delay $\tau_{avg}$ as \vspace{-2mm}
	\begin{align}
	\tau_{avg} = \frac{T + 1}{T + |C_T| - 1}  \bar\tau_C ~~\stackrel{T > |C_T|}{=} ~~ \cO\!\left( \bar\tau_C\right) \,.
	\end{align}
\end{restatable}

We explain this observation on a simple example. Assume that the concurrency is constant at every step ($\tau_C = \bar \tau_C$), and that all workers except one are responding very rarely. Then on steps 4--5 of Algorithm~\ref{alg:async-homogeneous-general} only this one responding worker would mostly participate. This means that for this one worker the delay $\tau_t$ would be frequently equal to zero, and the overall average delay will be small. %

Next, we provide our theoretical results. We first focus on the Asynchronous SGD Algorithm~\ref{alg:async-homogeneous-general} under constant stepsizes, i.e. $\eta_t \equiv \eta$. This setting was studied in many works such as \cite{Agarwal2011:delayed, Feyzmahdavian2016:async, arjevani20:delayedSGD, lian2015asynchronous, Stich20:error-feedback}

\begin{restatable}[Constant stepsizes]{theorem}{firstthm}\label{thm:homogeneous}
	Under Assumptions \ref{a:stoch_noise}, \ref{a:lsmooth_nc}, there exists a constant stepsize $\eta_t \equiv \eta$ such that for Algorithm~\ref{alg:async-homogeneous-general} it holds that $\frac{1}{T + 1} \sum_{t = 0}^T \norm{\nabla f(\xx^{(t)})}_2^2 \leq \epsilon$ after
	\begin{align}\label{eq:thm_first_part}
	\cO\!\left(\frac{\sigma^2}{\epsilon^2}  + \frac{\sqrt{\tau_{C} \tau_{\max}}}{\epsilon}\right) && \text{iterations}.
	\end{align}
	If we additionally assume bounded gradient Assumption~\ref{a:bounded_gradient}, then $\frac{1}{\sum_{t = 0}^T |\cA_t|}\sum_{t = 0}^T |\cA_t| \norm{\nabla f({\xx}^{(t)})}_2^2  \leq \epsilon$ after
	\begin{align}\label{eq:thm_bounded_delay}
	\cO\!\left(\frac{\sigma^2}{\epsilon^2} + \frac{\tau_{C} G}{\epsilon^{3/2}} + \frac{\tau_{C}}{\epsilon}\right) && \text{iterations.}
	\end{align}
\end{restatable}
Under constant concurrency, we can directly connect $\tau_C$ to the average delay $\tau_{avg}$ due to Remark~\ref{rem:conc_avg_delay}. We highlight again that in practice, to get the best utilization of the available resources, practical implementations choose the maximum concurrency possible, which is equal to $n$. 
\begin{corollary}
	If in Algorithm~\ref{alg:async-homogeneous-general} the concurrency is constant at every step (thus $\tau_C = \bar\tau_{C}$), then under the same conditions as in Theorem~\ref{thm:homogeneous} the convergence rate of Algorithm~\ref{alg:async-homogeneous-general} is 
	\begin{align}
	\cO\!\left(\frac{\sigma^2}{\epsilon^2}  + \frac{\sqrt{\tau_{avg} \tau_{\max}}}{\epsilon}\right) &&  \text{and} && \cO\!\left(\frac{\sigma^2}{\epsilon^2} + \frac{\tau_{avg} G}{\epsilon^{3/2}} + \frac{\tau_{avg}}{\epsilon}\right) 
	\end{align}
	for the case without and with bounded gradient Assumption~\ref{a:bounded_gradient} correspondingly. 
\end{corollary}
The previously best known convergence rate for Asynchronous SGD~\ref{alg:async-homogeneous-general} under constant stepsizes was given in \cite{Stich20:error-feedback} and is equal to $\cO\big(\frac{\sigma^2}{\epsilon^2}  + \frac{\tau_{\max}}{\epsilon}\big)$. In our theorem we improved the delay dependence from~$\tau_{\max}$ to $\sqrt{\tau_{avg} \tau_{\max}}$ in the last term \emph{without any change in the algorithm}, only by taking into account concurrency that is usually fixed in practical implementations anyways. No other work previously made an assumption on the number of computing workers in their theoretical analysis. $\sqrt{\tau_{avg} \tau_{\max}}$ could be much smaller than $\tau_{\max}$ in the presence of rare straggler devices. With an additional assumption of bounded gradients, the dependence on the maximum delay can be completely removed.

\subsection{Theoretical analysis: Delay-adaptive stepsizes}
In many cases, the bounded gradient Assumption~\ref{a:bounded_gradient} is unrealistic \cite{nguyen18:SGD_and_hogwild_without_bounded_gradient}, meaning that the gradient bound $G$ is often large and thus the rate \eqref{eq:thm_bounded_delay} is loose. 
In this section we show that by weighting the stepsize down for the gradients that have a large delay, once can remove the dependence on the maximum delay $\tau_{\max}$ without assuming bounded gradients (Assump.~\ref{a:bounded_gradient}).

\begin{restatable}[Delay-adaptive stepsizes]{theorem}{secondthm}\label{thm:homogeneous-adaptive} There exist a parameter $\eta \leq \frac{1}{4L}$ such that if we set the stepsizes in Algorithm~\ref{alg:async-homogeneous-general} dependent on the delays as
	\begin{align}\label{eq:adaptive_stepsizes}
	\eta_t = \begin{cases}
	\eta & \tau_t \leq \tau_{C}, \\
	< \min\{\eta, \frac{1}{4 L \tau_t} \}\vspace{-1mm} & \tau_t > \tau_C ,
	\end{cases}
	\end{align}
	then for Algorithm~\ref{alg:async-homogeneous-general}, under Assumptions~\ref{a:stoch_noise}, \ref{a:lsmooth_nc} it holds that $\frac{1}{\sum_{t = 0}^T \eta_t} \sum_{t = 0}^T \eta_t \norm{\nabla f(\xx^{(t)})}_2^2 \leq \epsilon$ after\vspace{-2mm}
	\begin{align}
	\cO\!\left(\frac{\sigma^2}{\epsilon^2} +\frac{{\tau_C}}{\epsilon}\right) && \text{iterations.}
	\end{align}
\end{restatable}
In our theorem, the stepsize $\eta_t$ in the case of large delays $\tau_t > \tau_C$ can be an arbitrary value between~$0$ and $\min\{\eta, \frac{1}{4 L \tau_t} \}$. Setting the stepsize $\eta_t \equiv 0$ is equivalent to dropping these gradients. 

\begin{proof}[Proof sketch of Theorem~\ref{thm:homogeneous-adaptive}]
	We give the intuitive proof sketch for the case when we drop gradients with $\tau_t > \tau_C$ and we deal with the general case in the Appendix. We know that $\tau_{avg} \approx \bar\tau_{C} \leq \tau_C$ from Remark~\ref{rem:conc_avg_delay}. It also holds that the number of gradients that have delay larger than the average delay $\tau_{avg}$ is smaller than half of all the gradients ($\leq \frac{T}{2}$) because delays are bounded below by zero ($\tau_t \geq 0~\forall t$). Thus, dropping the gradients with the delay $\tau_t > \tau_C$, or equivalently setting their stepsize $\eta_t \equiv 0$, will degrade the convergence rate at most by half, while the maximum delay among the applied ones now is equal to $\tau_C$. Thus we can apply result from \cite{Stich20:error-feedback} with $\tau_{max} = \tau_C$. 
\end{proof}
\begin{corollary}
	If in Algorithm~\ref{alg:async-homogeneous-general} the concurrency is constant at every step (thus $\tau_C = \bar\tau_{C}$), then under the same conditions as in Theorem~\ref{thm:homogeneous-adaptive} the convergence rate of Algorithm~\ref{alg:async-homogeneous-general} is equal to 
		\begin{align}\label{eq:rate-homo-avg}
	\cO\!\left(\frac{\sigma^2}{\epsilon^2} +\frac{{\tau_{avg}}}{\epsilon}\right).
	\end{align}\vspace{-1em}
\end{corollary}

\subsection{Discussion}
\paragraph{Comparison to synchronous optimization.}
Mini-batch SGD with batch size $n$ has the same degree of parallelism as Algorithm~\ref{alg:async-homogeneous-general} with constant concurrency $n$, i.e. it has $n$ workers computing gradients in parallel. Mini-batch SGD needs $\cO\big(\frac{\sigma^2}{n \epsilon^2} + \frac{1}{\epsilon}\big)$ \cite{Ghadimi2013:SGDrate} batches of gradients to reach an $\epsilon$-stationary point, and thus needs $\cO\big(\frac{\sigma^2}{\epsilon^2} + \frac{n}{\epsilon}\big)$ gradients, as the batch-size is equal to $n$. On the contrary, asynchronous SGD Algorithm~\ref{alg:async-homogeneous-general} with stepsizes chosen as in \eqref{eq:adaptive_stepsizes} achieves exactly the same rate \eqref{eq:rate-homo-avg} since $\tau_{avg}=\tau_C=n$, while its expected per-iteration time is \emph{faster} than that of mini-batch SGD, as no workers have to wait for others. Thus, our result shows that asynchronous SGD is always faster than mini-batch SGD \emph{regardless of the delay pattern}. A small note that in our reasoning we implicitly assumed that the sever can perform its operations in negligible time.

\paragraph{Tuning the stepsize.}
It is worth noting that our stepsize rule \eqref{eq:adaptive_stepsizes} does not introduce any additional hyperparameters to tune compared to the constant stepsize case or to synchronous SGD. $\tau_C$ is usually known and can be easily controlled by the server, especially in the practical constant concurrency case. Thus, to implement such a stepsize rule \eqref{eq:adaptive_stepsizes} one needs to tune only stepsize $\eta$, and in case of $\tau_t>\tau_C$ set stepsize $\eta_t \leq \frac{\eta}{\tau_t}$. 

\paragraph{Average v.s.\ maximum delay. } In a homogeneous environment when every worker computes gradients with same speed during the whole training, the average and maximum delays would be almost equal. However, occasional straggler devices will usually be present. In this case the maximum delay is much larger than the average delay. 

Consider a simple example with $n = 2$ workers, where the first worker computes gradients very fast, while the second worker returns its gradient only at the end of the training at the last iteration $T$. In this case the average delay $\tau_{avg} = 2$ is a small constant, while the maximum delay $\tau_{\max} = T$. In this case the rate depending only on the maximum delay $\tau_{\max}$ would guarantee convergence only up to a constant accuracy $\varepsilon = \cO(1)$. While both rates with $\sqrt{\tau_{\max} \tau_{avg}}$ and with $\tau_{avg}$ guarantee convergence up to an arbitrary small accuracy.

\paragraph{Comparison to other methods. }
\citet{Cohen2021_pickySGD} recently proposed the PickySGD algorithm that achieves a convergence rate of $\cO\big( \frac{\sigma^2}{\epsilon^2} +\frac{{\tau_{avg}}}{\epsilon} \big)$ (same as \eqref{eq:rate-homo-avg}). Their algorithm discards gradients based on the distance between the current point and the delayed one $\norm{\xx^{(t)} - \xx^{(t - \tau_t)}}$. The disadvantage of their method is that it requires sending points $\xx^{(t - \tau_t)}$ along with the gradients thus incurring twice more communications at every step. Their method also requires tuning an extra hyperparameter. In this work we achieve the same convergence rate with a much simpler method that does not require any additional communications nor additional tuning compared to synchronous SGD.

\cite{Aviv21:delayed_average} also recently proposed the delay-adaptive algorithm with convergence rate depending on the average delay $\tau_{avg}$ for the convex and strongly convex cases. Although, our convergence rates are for the non-convex case and are not directly comparable to theirs, we highlight some key differences in their analysis. First, their convergence rate depends not only on $\tau_{avg}$ but also on the variance $\sigma_{\tau}$ of the delays, which can degrade with the maximum delay. Second, they require the bounded gradient Assumption~\ref{a:bounded_gradient}. In Theorem~\ref{thm:homogeneous} we show that under Assumption~\ref{a:bounded_gradient} no modifications to the algorithm are needed to completely remove the dependence on the maximum delay $\tau_{\max}$ \eqref{eq:thm_bounded_delay}. 

\paragraph{Tightness. } 
As we explained in Example~\ref{ex:minibatch}, mini-batch SGD is covered by Algorithm~\ref{alg:async-homogeneous-general}. %
We know that mini-batch SGD convergence is lower bounded by $\Theta\big(\frac{\sigma^2}{n\epsilon^2} + \frac{1}{\epsilon}\big)$  \cite{Arjevani2019:LowerBoundSGD} in terms of batches processed and thus by $\Theta\big(\frac{\sigma^2}{\epsilon^2} + \frac{n}{\epsilon}\big)$ in terms of the gradients computed. Our convergence rate given in Theorem~\ref{thm:homogeneous} \emph{coincides with this lower bound} as in this case concurrency $\tau_{C} = n$, $\tau_{avg} = \bar\tau_C = \frac{n}{2}$.

\section{Heterogeneous Distributed Setting}\label{sec:heterogeneous}
In this section we consider more general problems of the form \eqref{eq:problem} where the functions $f_i$ are different on different nodes.
This setting is motivated for example by federated learning \cite{McMahan16:FedLearning,Kairouz2019:federated}, where every node (client) possesses its own private data, possibly coming from a different data distribution, and thus has its own different local loss function $f_i$. 

The setting here is therefore more general than the one considered in previous Section~\ref{sec:homogeneous}, and we will see that some of the results (with the constant stepsizes) in the homogeneous case follow as a special case of the more general results we present in this section.

\subsection{Algorithm}

We consider asynchronous SGD as given in Algorithm~\ref{alg:async-general}. Close variants of this algorithm were studied in several prior works \cite{Nguyen22:FedBuff, Stich2018:LocalSGD}. In order to simplify the presentation, we consider that concurrency is constant over time (and thus $\tau_C = \bar\tau_C$ in Definition~\ref{def:concurrency}). In order to allow for client subsampling often implemented in practical FL applications, we allow the concurrency $\tau_C$ to be smaller than overall number of workers $n$. The same concurrency model was recently considered in the practical FedBuff algorithm \cite{Nguyen22:FedBuff}.

\begin{algorithm}[ht]
	\caption{\textsc{Asynchronous SGD} with concurrency $\tau_{C}$}\label{alg:async-general}
	\let\oldfor\algorithmicfor
	\begin{algorithmic}[1]
		\INPUT{Initial value $\xx^{(0)} \in \R^d$, $n$ clients, concurrency $\tau_{C}$}\\[1ex]
		\textbf{Server: }\\
		\STATE sever selects \emph{uniformly at random} a set of active clients $\cC_0$ of size $\tau_C$ and sends them $\xx^{(0)}$
		\FOR{$t=0,\dots, T-1$}
		\STATE active clients $\cC_t$ are computing stochastic gradients in parallel at the assigned points
		\STATE once some client $j_t$ finishes compute, it sends $\nabla F_{j_t}(\xx^{(t - \tau_t)}, \xi_t)$ to the server
		\STATE server updates $\xx^{(t + 1)} = \xx^{(t)} - \eta_t \nabla F_{j_t}(\xx^{(t - \tau_t)}, \xi_t)$ 
		\STATE sever selects a new client $k_t \sim \operatorname{Uniform}[1, n]$ and sends it $\xx^{(t + 1)}$
		\STATE update the active worker multiset $\cC_{t + 1} = \cC_t \backslash \{j_t \} \cup \{k_t\}$
		\ENDFOR
	\end{algorithmic}
\end{algorithm}
The algorithm is very similar to the homogeneous Algorithm~\ref{alg:async-homogeneous-general} with two key differences: at line 6, the server selects clients \emph{out of all clients}, and does so \emph{uniformly at random}, regardless of the current active worker set $\cC_t$. This means that the same client can get sampled several times, even if it didn't finish its previous job(s) yet (thus $\cC_t$ is a multiset). In this case, the assigned jobs would just pile up on this client. 

\subsection{Theoretical analysis}
We first note that our key observation on the delays (Remark \ref{rem:conc_avg_delay}) holds for Algorithm~\ref{alg:async-general} as well. Moreover, as we have a constant concurrency $\tau_C$ at every step, $\tau_{avg} = \cO\!\left(\tau_C\right)$. 

\begin{definition} \label{def:avg_delay_i}
	Denote a (possibly empty) set $\{\tau_{k}^{C_T, i} \}_{k}$ to be the set of delays from gradients of the client $i$ that are left unapplied at the last iteration of the Algorithm~\ref{alg:async-general}.
	
	We define the average delay of a client $i$ as 
	\begin{align*}
		\tau_{avg}^i = \frac{1}{T_i} \left( \sum_{t~ : ~j_t = i} \tau_t + \sum_{k} \tau_{k}^{C_T, i} \right)
	\end{align*}
	where $T_i$ is the number of times the client $i$ got sampled during lines 1 and 6 of Algorithm~\ref{alg:async-general}. 
\end{definition}
\begin{assumption}\label{a:avg_dely}
	The average delay $\tau_{avg}^i$ is independent from the number of times $T_i$ the client $i$ got sampled. 
\end{assumption}

\begin{restatable}[constant stepsizes]{theorem}{themheterogeneous}\label{thm:heterogeneous}
	Under Assumptions \ref{a:stoch_noise}, \ref{a:heterogeneity}, \ref{a:lsmooth_nc}, \ref{a:avg_dely} there is exist a constant stepsize $\eta_t \equiv \eta$ such that for Algorithm~\ref{alg:async-general} it holds that $\frac{1}{T + 1} \sum_{t = 0}^T \norm{\nabla f(\xx^{(t)})}_2^2 \leq \epsilon$ after
	\begin{align}\label{eq:thm_het_first}
	\cO\!\left(\frac{\sigma^2}{\epsilon^2} + \frac{\zeta^2}{\epsilon^2} + \frac{\sqrt{\tau_{avg} \frac{1}{n} \sum_{i = 1}^n \zeta_i^2 \tau^i_{avg}}}{\epsilon^{\frac{3}{2}}}+ \frac{\sqrt{\tau_{avg} \tau_{\max}}}{\epsilon}\right) && \text{iterations,}
	\end{align}
	
	Under Assumptions~ \ref{a:stoch_noise}, \ref{a:heterogeneity}, \ref{a:lsmooth_nc} and additional bounded gradient Assumption~\ref{a:bounded_gradient}, it holds that $\frac{1}{T + 1} \sum_{t = 0}^T \norm{\nabla f(\xx^{(t)})}_2^2 \leq \epsilon$ after
	\begin{align}\label{eq:het_bounded}
	\cO\!\left(\frac{\sigma^2}{\epsilon^2} + \frac{\zeta^2}{\epsilon^2}  + \frac{\tau_{avg} G}{\epsilon^{\frac{3}{2}}} + \frac{\tau_{avg}}{\epsilon}\right) && \text{iterations.}
	\end{align}	
\end{restatable}
We note that the leading $\frac{1}{\epsilon^2}$ term is affected by heterogeneity $\zeta^2$ because at every step we apply gradient from only one client. This term is usually present in the federated learning algorithms with client subsampling see e.g. \cite{Karimireddy19:scaffold}. 
\subsection{Discussion}
\paragraph{Comparison to other works. }
The recent FedBuff algorithm \cite{Nguyen22:FedBuff} is similar to our Algorithm~\ref{alg:async-general}. Their algorithm allows clients to perform several local steps and the server to wait for more than 1 client to finish compute (aka buffering), which we did not include for simplicity as these aspects are orthogonal to the effect of delays. 

Disregarding these two orthogonal changes, the FedBuff algorithm is almost equivalent to our Algorithm~\ref{alg:async-general} with a key difference: they assume that the client $j_t$ that finishes computation at every step comes from the uniform distribution over all the clients. This is unrealistic to assume in practice because the server cannot control which clients finish computations at every step. In Algorithm~\ref{alg:async-general} we have the more realistic assumption only on the sampling process of the clients (on line 6) that \emph{can be controlled} by the server. This reflects practical client sampling in federated learning. 

The convergence rate of FedBuff \cite{Nguyen22:FedBuff} under the bounded gradient assumption is $\cO\!\left( \frac{\sigma^2}{\epsilon^2} + \frac{\zeta^2}{\epsilon^2} + \frac{(\zeta^2 + 1) \tau_{\max} G^2}{\epsilon} \right)$. In contrast, in Theorem~\ref{thm:heterogeneous} we completely remove the dependence on the maximum delay $\tau_{\max}$ under bounded gradients (as in Equation \eqref{eq:het_bounded}).

\paragraph{Delays.} We note that for Theorem~\ref{thm:heterogeneous} we did not impose any assumption on the delays. Thus, our result allows clients and the delays on these clients to be dependent, meaning that some of the clients could be systematically slower than others. Interestingly, the middle heterogeneity term (the term with $\zeta_i$) is not affected by the maximum delay at all, but is affected by the average delay within each individual client. If all the heterogeneity parameters are equal, i.e. $\zeta_i = \zeta_j, \forall i,j$, then the middle term will be affected only by the overall average delay $\tau_{avg}$. 

\paragraph{Gradient clipping. } Practical implementations of FL algorithms usually apply clipping to the gradients in order to guarantee differential privacy \cite{Kairouz2019:federated}%
. This automatically bounds the norm of all applied gradients, making the the constant $G^2$ in Assumption~\ref{a:bounded_gradient} small. Although we do not provide formal convergence guarantees of asynchronous SGD with gradient clipping, we envision that its convergence rate would depend only on the average delay, similar to the bounded gradient case \eqref{eq:thm_bounded_delay}, thus making the algorithm robust to stragglers. 

\paragraph{Delay-adaptive stepsizes.} For homogeneous functions we have shown that delay-adaptive stepsizes result in a convergence rate dependent only on the average delay $\tau_{avg}$ without assuming bounded gradients (as in Equation \eqref{eq:adaptive_stepsizes}). However in the heterogeneous case this is not so straightforward. Delay-adaptive learning rate schemes will introduce a bias towards the clients that compute quickly, and Algorithm~\ref{alg:async-general} would converge to the wrong objective.

It is interesting to note that current popular schemes implemented in practice for FL over-selects the clients at every iteration \cite{Bonawitz19:over-selection}. The server waits only for some percentage (e.g. 80\%) of sampled clients and discards the rest. Such a scheme also introduces a bias towards fast workers. A delay-adaptive learning rate scheme is expected to introduce less bias as the gradients are still applied but with the smaller weight. We leave this question for future practical investigations, as it is not the focus of our current work.

\paragraph{Independent delays. } If the delays and the clients are independent (e.g. coming from the same distribution for all of the clients), then the convergence rate of Algorithm~\ref{alg:async-general} will simplify to $\cO\!\left(\frac{\sigma^2}{\epsilon^2}  + \frac{\zeta\tau_{avg}}{\epsilon^{\frac{3}{2}}}+ \frac{\sqrt{\tau_{avg} \tau_{\max}}}{\epsilon}\right)$ (without needing bounded gradient assumption). 
In this case it is also possible to use delay-adaptive stepsizes (similar to Theorem~\ref{thm:homogeneous-adaptive}) to completely remove the dependence on the maximum delays $\tau_{\max}$ without assuming bounded gradients. %

\paragraph{Extensions.}
We can extend the Algorithm~\ref{alg:async-general} and our theoretical analysis to allow clients to perform several local steps, before sending back the change in $\xx$. We can also extend Algorithm~\ref{alg:async-general} to allow the server to wait for the first $K$ clients to finish computations rather than just one, similar to \cite{Nguyen22:FedBuff}. These extensions are straightforward and we excluded them here for simplicity of presentation. 

Finally, we can also extend Algorithm~\ref{alg:async-general} to sample new clients as soon as some previous client finished compute, without waiting for the server update on the line 5. %

\subsection{Estimating Speedup over Synchronous SGD}
Assume we have $n$ clients, each of which having a different but constant time to compute a gradient $\{\Delta_i\}_{i = 1}^n$. W.l.o.g. we assume that $\Delta_i$ are ordered as $\Delta_1 \leq \Delta_2 \leq \dots \leq \Delta_n$. 

\begin{lemma}\label{lem:speedup}
	In expectation, the asynchronous Algorithm~\ref{alg:async-general} needs \vspace{-1mm}
	\begin{align*}
	\bar \Delta = \frac{1}{n} \sum_{ i = 1}^n \Delta_i 
	\end{align*}
	time to compute $\tau_C$ gradients, while mini-batch SGD with batch size $\tau_C$ needs \vspace{-1mm} 
	\begin{align*}
	\tilde \Delta = \sum_{ i = 1}^n \alpha_i \Delta_i
	\end{align*}
	time to compute a batch of $\tau_C$ gradients, where $\alpha_i = \frac{i^{\tau_C}  - (i - 1)^{\tau_C}}{n^{\tau_C}}$. 
	It is also always holds that 
	$\bar\Delta \leq \tilde \Delta$. 
\end{lemma}
With this lemma we can precisely estimate how much faster the asynchronous algorithm is compared to the classic synchronous mini-batch one. Note that $\alpha_i$ are increasing with $i$ with a rate of $\cO(i^{\tau_C})$, thus in mini-batch SGD, the large delays get a much higher weight than the small delays, especially when the batch size $\tau_C$ is large. 

For example, consider $1000$ clients, $900$ of which compute their update every $10$s, while $100$ of them computes their update every $60s$. Then the expected time for $\tau_C$ gradients of the asynchronous algorithm will be $15$s, while synchronous mini-batch SGD (with $\tau_C = 10$) will take a significantly longer time of $42.5$s for the same number of gradients.

\section{Conclusion}

In this paper we study the asynchronous SGD algorithm both in homogeneous and heterogeneous settings. By leveraging the notion of concurrency---the number of workers that compute gradients in parallel---we show a much faster convergence rate for asynchronous SGD, improving the dependence on the maximum delay $\tau_{\max}$ over  prior works, for both homogeneous and heterogeneous objectives. Our proof technique also allows to design a simple delay-adaptive stepsize rule~\eqref{eq:adaptive_stepsizes} that attains a convergence rate depending only on the average delay $\tau_{avg}$ that neither requires any additional tuning, nor additional communication. Our techniques allows us to demonstrate that \emph{asynchronous SGD is faster than mini-batch SGD for any delay pattern}. 

\section*{Acknowledgments}
AK was supported by a Google PhD Fellowship. 
We thank Brendan McMahan, Thijs Vogels, Hadrien Hendrikx, Aditya Vardhan Varre and Maria-Luiza Vladarean for useful discussions and their feedback on the manuscript. 

\medskip

\bibliographystyle{myplainnat} %
{\small
\bibliography{reference}
}

\newpage
\appendix
\section{Proofs}
In this section we provide the proofs of all the theoretical results stated in the main paper.
\subsection{Proof of Remark~\ref{rem:conc_avg_delay}}

First, we prove our key observation given in Remark~\ref{rem:conc_avg_delay}. 
\remarkkey*

\begin{proof}
	Define $\{\tau_i^{\cC_t} \}_{i \in \cC_t}$ as the set of delays of the gradients that are left in the active worker set before iteration $t$ is performed, i.e. each $\tau_i^{\cC_t}$ is equal to the difference between the current iteration $t$ and the iteration at which worker $i$ started to compute its current gradient for $t > 0$, and $\tau_i^{\cC_0} = 1$ for all $i \in \cC_0$, that is the initial set of active workers. For simplicity we denote 
	\begin{align*}
	\tau^{\text{active}, t}_{sum} := \sum_{i \in \cC_t} \tau_i^{\cC_t} \,.
	\end{align*}
	We also define $\tau_{sum}^{\text{applied}, t}$ as the sum of all delays of gradients applied before iteration $t$ is performed, i.e. 
	\begin{align*}
	\tau_{sum}^{\text{applied}, t} := \sum_{j = 0}^{t - 1} \tau_j \,.
	\end{align*}
	
	At the zero-th iteration we have that 
	\begin{align} \label{eq:rem5_initial_cond}
	\tau_{sum}^{\text{applied}, 0}  = 0\,, && \tau^{\text{active}, 0}_{sum} = \tau_C^{(0)}\,,
	\end{align}
	as no gradients were applied yet. 
	
	We claim that 
	\begin{align}\label{eq:rem5_recursion}
	\tau_{sum}^{\text{applied}, t + 1}  + \tau^{\text{active}, t + 1}_{sum} = \tau_{sum}^{\text{applied}, t }  + \tau^{\text{active}, t }_{sum} + \tau_{C}^{(t + 1)} \,.
	\end{align}
	Indeed, one of the gradients from $\cC_t$ got applied and its delay moved from $\tau^{\text{active}, t }_{sum}$ to $\tau_{sum}^{\text{applied}, t + 1}$. The newly selected active workers in line 6 of Algorithm~\ref{alg:async-homogeneous-general} have delay zero, as they just started their computations in this step. And all of the current active workers in $\cC_{t + 1}$ (of size $|\cC_{t + 1}| = \tau_C^{(t + 1)}$) got an increase by 1 due to increase of the iteration count from $t$ to $t + 1$. 
	
	Using the initial conditions \eqref{eq:rem5_initial_cond} and \eqref{eq:rem5_recursion} we can conclude that 
	\begin{align*}
	\tau_{sum}^{\text{applied}, T}  + \tau^{\text{active}, T}_{sum} = \sum_{t = 0}^{T}\tau_{C}^{(t)} = (T + 1) \bar \tau_{C} \,.
	\end{align*}
	Note that the left hand side is exactly equal to $(T + |C_T| - 1) \tau_{avg}$ from our Definition~\ref{def:delays}. 
	Thus, 
	\begin{align*}
	\tau_{avg}= \frac{T + 1}{T + |C_T| - 1} \bar \tau_{C} = \cO\left(\bar\tau_C\right) \,,
	\end{align*}
	where the last equality holds if $T > |C_T|$. 
\end{proof}

\subsection{Useful inequalities}
\begin{lemma}\label{remark:norm_of_sum}
	For an arbitrary set of $n$ vectors $\{\aa_i\}_{i = 1}^n$, $\aa_i \in \R^d$
	\begin{equation}\label{eq:norm_of_sum}
	\norm{\sum_{i = 1}^n \aa_i}^2 \leq n \sum_{i = 1}^n \norm{\aa_i}^2 \,.
	\end{equation}
\end{lemma}

\subsection{Proof of Theorems~\ref{thm:homogeneous}, \eqref{eq:thm_first_part} and \ref{thm:homogeneous-adaptive}}
We first recall both of the theorems
\firstthm*

\secondthm*
We first give a common lemma that will be used in the proofs for both of the theorems.
\begin{lemma}[Descent Lemma]\label{lem:descent}
	Under Assumptions~\ref{a:stoch_noise} and \ref{a:lsmooth_nc}, if in Algorithm~\ref{alg:async-homogeneous-general} the stepsize $\eta_t < \frac{1}{2 L }$ then it holds that 
	\begin{align*}
	\EE{t + 1}{f({\xx}^{(t + 1)})} &\leq f({\xx}^{(t)}) - \frac{\eta_t}{2} \norm{\nabla f({\xx}^{(t)})}_2^2- \frac{\eta_t}{4}\norm{\nabla f({\xx}^{(t - \tau_t)})}^2 + L \eta_t^2 \sigma^2 + \frac{\eta_t L^2}{2} \norm{{\xx}^{(t)} -  \xx^{(t - \tau_t)}}_2^2 ,.
	\end{align*}
\end{lemma}
\begin{proof}
	Because the function $f$ is $L$-smooth, we have 
	\begin{align*}
	\EE{t + 1}{f({\xx}^{(t + 1)})} &= \EE{t + 1}{f\left({\xx}^{(t)} - \eta_t \nabla F(\xx^{(t - \tau_t)}, \xi_t )\right) }\\
	& \leq f({\xx}^{(t)}) - \eta_t \underbrace{\EE{t + 1}{\lin{ \nabla f({\xx}^{(t)}),   \nabla F(\xx^{(t - \tau_t)}, \xi_t )}} }_{=:T_1}+ \EE{t + 1}{\frac{L}{2} \eta_t^2 \underbrace{ \norm{\nabla F(\xx^{(t - \tau_t)}, \xi_t )}_2^2}_{=:T_2} }\\
	\end{align*}
	We first estimate the second term as 
	\begin{align*}
	T_1 &= - \eta_t  \lin{ \nabla f({\xx}^{(t)}),  \nabla f(\xx^{(t - \tau_t)})} = - \frac{\eta_t}{2}\norm{\nabla f({\xx}^{(t)})}^2 - \frac{\eta_t}{2}\norm{\nabla f({\xx}^{(t - \tau_t)})}^2 + \dfrac{\eta_t}{2} \norm{\nabla f({\xx}^{(t)}) - \nabla f(\xx^{(t - \tau_t)})}^2
	\end{align*}
	For the last term, we add and subtract $\nabla f(\xx^{(t - \tau_t)})$, and use that $\E_{t + 1}\nabla F(\xx^{(t - \tau_t)}, \xi_t ) - \nabla f(\xx^{(t - \tau_t)}) = 0$
	\begin{align*}
	T_2 & =  \EE{t + 1}{\norm{\nabla F(\xx^{(t - \tau_t)}, \xi_t ) - \nabla f(\xx^{(t - \tau_t)})}_2^2} + \norm{\nabla f(\xx^{(t - \tau_t)})}^2_2 \\
	& \stackrel{\eqref{eq:stochastic_noise}}{\leq }  \sigma^2  + \norm{\nabla f(\xx^{(t - \tau_t)})}^2_2 \,. 
	\end{align*}
	Combining this together and using $L$-smoothness to estimate $\norm{\nabla f({\xx}^{(t)}) - \nabla f(\xx^{(t - \tau_t)})}_2^2$,
	\begin{align*}
	\EE{t + 1}{f({\xx}^{(t + 1)})} &\leq f({\xx}^{(t)}) - \eta_t \norm{\nabla f({\xx}^{(t)})}_2^2 - \frac{\eta_t}{2} \left(1 - L\eta_t\right) \norm{\nabla f({\xx}^{(t - \tau_t)})}_2^2 + \frac{\eta_t L^2}{2}  \norm{{\xx}^{(t)} -  \xx^{(t - \tau_t)}}_2^2 \\
	&+  L \eta_t^2 \sigma^2 \,.
	\end{align*}
	Applying $\eta < \frac{1}{2 L }$ we get statement of the lemma.
\end{proof}

\subsubsection{Proof of Theorem~\ref{thm:homogeneous}, convergence rate \eqref{eq:thm_first_part}}
\begin{lemma}[Estimation of the residual]\label{lem:residual}
	Under Assumptions~\ref{a:stoch_noise} and \ref{a:lsmooth_nc}, the iterates of Algorithm~\ref{alg:async-homogeneous-general} with the constant stepsize  $\eta_t \equiv \eta$ with $\eta \leq \frac{1}{2L \sqrt{\tau_{\max} \tau_C}}$ satisfy
	\begin{align*}
	\frac{1}{T + 1}\sum_{t = 0}^T \E \norm{{\xx}^{(t)} -  \xx^{(t - \tau_t)}}_2^2 \leq \frac{1}{4 L^2 (T + 1)} \sum_{t = 0}^{T } \E\norm{\nabla f(\xx^{(t - \tau_t)})}^2 + \frac{\sigma^2 \eta}{2 L } \,.
	\end{align*}
\end{lemma}
\begin{proof}
	We start with unrolling the difference and use that $\E \nabla F(\xx^{(j - \tau_j)}, \xi^{(j - \tau_j)}) = \nabla f (\xx^{(j - \tau_j)})$. 
	\begin{align*}
	\E \norm{{\xx}^{(t)} -  \xx^{(t - \tau_t)}}_2^2 &= \E \norm{\sum_{j = t - \tau_t}^{t -1}\eta \nabla F(\xx^{(j - \tau_j)}, \xi_j ) }^2 \stackrel{\eqref{eq:bounded_gradient}}{\leq}  \E \norm{\sum_{j = t - \tau_t}^{t -1}\eta \nabla f(\xx^{(j - \tau_j)}) }^2 + \tau_t \eta^2 \sigma^2\\
	&\stackrel{\eqref{eq:norm_of_sum}}{\leq } \tau_t \E\sum_{j = t - \tau_t}^{t - 1} \eta^2 \norm{\nabla f(\xx^{(j - \tau_j)})}^2 + \tau_t \eta^2 \sigma^2 \,.
	\end{align*}
	Using that $\eta \leq \frac{1}{2L \sqrt{\tau_{\max} \tau_C}}$,
	\begin{align*}
	\E \norm{{\xx}^{(t)} -  \xx^{(t - \tau_t)}}_2^2 \leq \frac{1}{4 L^2 \tau_C} \sum_{j = t - \tau_t}^{t - 1} \E\norm{\nabla f(\xx^{(j - \tau_j)})}^2 + \tau_t \eta^2 \sigma^2 \,.
	\end{align*}
	Summing over $T$, 
	\begin{align*}
	\sum_{t = 0}^T \E\norm{{\xx}^{(t)} -  \xx^{(t - \tau_t)}}_2^2 &\leq \frac{1}{4 L^2 \tau_C} \sum_{t = 0}^{T } \sum_{j = t - \tau_t}^{t - 1} \E  \norm{\nabla f(\xx^{(j - \tau_j)})}^2 + \sum_{t = 0}^T \tau_t \eta^2 \sigma^2 \\ 
	& \leq \frac{1}{4 L^2 \tau_C} \sum_{t = 0}^{T } \sum_{j = t - \tau_t}^{t - 1} \E  \norm{\nabla f(\xx^{(j - \tau_j)})}^2 + (T + 1)\tau_{avg} \eta^2 \sigma^2 \,.
	\end{align*}
	We now observe that the number of times each of the gradients $\norm{\nabla f(\xx^{(j - \tau_j)})}^2$ appears in the right hand side is bounded by $\tau_C^{(j)} - 1$ because this many gradients started to be computed before the iteration $j$ and will get applied at some iteration $t > j$. Thus, 
	\begin{align*}
	\sum_{t = 0}^T \E \norm{{\xx}^{(t)} -  \xx^{(t - \tau_t)}}_2^2 \leq \frac{1}{4 L^2} \sum_{t = 0}^{T } \E\norm{\nabla f(\xx^{(t - \tau_t)})}^2 + (T + 1)\frac{\sigma^2 \eta}{2 L } \,,
	\end{align*}
	where for the last $\sigma$ term we estimated $\eta \leq \frac{1}{2L \sqrt{\tau_{\max}\tau_C}}$ and used that both $\tau_{avg} \leq \tau_C$ and $\tau_{avg} \leq \tau_{\max}$. 
	Dividing the inequality by $T + 1$ we get the statement of the lemma. 
\end{proof}

Next, we give the proof of the first part of Theorem~\ref{thm:homogeneous}.
\begin{proof}[Proof of Theorem~\ref{thm:homogeneous}, convergence rate \eqref{eq:thm_first_part}]
	We start by averaging with $T$ and dividing by $\eta$ the descent Lemma~\ref{lem:descent}. 
	\begin{align*}
	\frac{1}{T + 1}  \sum_{t = 0}^T \left(\frac{1}{2} \E \norm{\nabla f({\xx}^{(t)})}_2^2 +  \frac{1}{4} \E \norm{\nabla f({\xx}^{(t - \tau_t)})}^2\right)  &\leq \frac{1}{\eta(T + 1)}\left(f({\xx}^{(0)}) - f^\star \right) + L \eta \sigma^2 \\
	&\qquad \qquad + \frac{1}{T + 1}\frac{L^2}{2} \sum_{t = 0}^T \E  \norm{{\xx}^{(t)} -  \xx^{(t - \tau_t)}}_2^2 \,.
	\end{align*}
	We next apply Lemma~\ref{lem:residual} to the last term and get
	\begin{align*}
		\frac{1}{T + 1}  \sum_{t = 0}^T \left(\frac{1}{2} \E \norm{\nabla f({\xx}^{(t)})}_2^2 +  \frac{1}{4} \E \norm{\nabla f({\xx}^{(t - \tau_t)})}^2\right)  &\leq \frac{1}{\eta(T + 1)}\left(f({\xx}^{(0)}) - f^\star \right) + L \eta \sigma^2 \\
	&\qquad \qquad + \frac{1}{8 (T + 1)} \sum_{t = 0}^{T } \E\norm{\nabla f(\xx^{(t - \tau_t)})}^2 + \frac{L \eta \sigma^2}{4}  \,.
	\end{align*}
	And thus, 
	\begin{align*}
	\frac{1}{T + 1}  \sum_{t = 0}^T \E \norm{\nabla f({\xx}^{(t)})}_2^2   &\leq \frac{2}{\eta(T + 1)}\left(f({\xx}^{(0)}) - f^\star \right) + 4 L \eta \sigma^2 \,.
	\end{align*}
	It is only left to choose a stepsize $\eta$. Similar to previous works \cite{Stich20:error-feedback}, we chose it as
	$$\eta = \min\left\{ \frac{1}{2L \sqrt{\tau_{\max} \tau_{C}}} ; \left(\frac{r_0}{2 L \sigma^2 (T + 1)}\right)^{\frac{1}{2}} \right\} \leq \frac{1}{2L \sqrt{\tau_{\max} \tau_{C}}} \,, $$
	where we defined $r_0 = f({\xx}^{(0)}) - f^\star $.
	With this choice of stepsize we indeed have that
	\begin{itemize}
		\item If $\frac{1}{2L \sqrt{\tau_{\max} \tau_{C}}} \leq \left(\frac{r_0}{2 L \sigma^2 (T + 1)}\right)^{\frac{1}{2}}$  then $\eta = \frac{1}{2L \sqrt{\tau_{\max} \tau_{C}}}$,  and
		\begin{align*}
		\frac{1}{T + 1}  \sum_{t = 0}^T \E \norm{\nabla f({\xx}^{(t)})}_2^2  \leq \frac{4 L r_0 \sqrt{\tau_{\max} \tau_C} }{T + 1} + \left(\frac{r_0}{2 L \sigma^2 (T + 1)}\right)^{\frac{1}{2}} 4 L \sigma^2 = \cO\left(\frac{\sigma}{\sqrt{T}}  + \frac{\sqrt{\tau_{\max}\tau_C}}{T}\right)
		\end{align*}
		\item Otherwise if $\frac{1}{2L \sqrt{\tau_{\max} \tau_{C}}} > \left(\frac{r_0}{2 L \sigma^2 (T + 1)}\right)^{\frac{1}{2}}$
				\begin{align*}
		\frac{1}{T + 1}  \sum_{t = 0}^T \E \norm{\nabla f({\xx}^{(t)})}_2^2  \leq 2 \left(\frac{8 L \sigma^2 r_0}{(T + 1)}\right)^{\frac{1}{2}} = \cO\left(\frac{\sigma}{\sqrt{T}}\right)
		\end{align*}
	\end{itemize}
\end{proof}

\subsubsection{Proof of Theorem~\ref{thm:homogeneous-adaptive}}
\begin{lemma}[Estimation of the residual]\label{lem:residual_adaptive}
	Under Assumptions~\ref{a:stoch_noise} and \ref{a:lsmooth_nc}, the iterates of Algorithm~\ref{alg:async-homogeneous-general} with the stepsizes $\eta_t$ chosen as in \eqref{eq:adaptive_stepsizes}, which we repeat here for readability
	\begin{align*}
	\eta_t = \begin{cases}
	\eta & \tau_t \leq \tau_{C}, \\
	< \min\{\eta, \frac{1}{4 L \tau_t} \}\vspace{-1mm} & \tau_t > \tau_C ,
	\end{cases}
	\end{align*}
	with $\eta \leq \frac{1}{4L \tau_C}$ satisfy
	\begin{align*}
	\sum_{t = 0}^T \eta_t \norm{{\xx}^{(t)} -  \xx^{(t - \tau_t)}}_2^2 \leq \frac{1}{16L^2} \sum_{t = 0}^{T } \eta_t \norm{\nabla f(\xx^{(t - \tau_t)})}^2 + \frac{\sigma^2}{4L} \sum_{t = 0}^{T} \eta_t^2 \,.
	\end{align*}
\end{lemma}
\begin{proof}
	\begin{align*}
	\eta_t \norm{{\xx}^{(t)} -  \xx^{(t - \tau_t)}}_2^2 &=  \eta_t\norm{\sum_{j = t - \tau_t}^{t -1}\eta_j \nabla F(\xx^{(j - \tau_j)}, \xi_j ) }^2 \stackrel{\eqref{eq:stochastic_noise}}{\leq} \eta_t\norm{\sum_{j = t - \tau_t}^{t -1}\eta_j \nabla f(\xx^{(j - \tau_j)} )}^2  + \eta_t \sum_{j = t - \tau_t}^{t -1} \eta_j^2 \sigma^2\\
	 & \stackrel{\eqref{eq:norm_of_sum}}{\leq }\eta_t \tau_t \sum_{j = t - \tau_t}^{t - 1} \eta_j^2 \norm{\nabla f(\xx^{(j - \tau_j)})}^2 + \eta_t \sum_{j = t - \tau_t}^{t -1} \eta_j^2 \sigma^2 \,.
	\end{align*}
	We use that each of the stepsizes $\eta_t \leq \frac{1}{4 L \max\{ \tau_t, \tau_C \}}$. Thus, 
	\begin{align*}
	\eta_t \norm{{\xx}^{(t)} -  \xx^{(t - \tau_t)}}_2^2 \leq \frac{1}{4L} \sum_{j = t - \tau_t}^{t - 1} \eta_j^2 \norm{\nabla f(\xx^{(j - \tau_j)})}^2 + \frac{1}{4 L \tau_C} \sum_{j = t - \tau_t}^{t -1} \eta_j^2 \sigma^2 \,.
	\end{align*}
	Summing over $T$, and using that each of the gradients $\norm{\nabla f(\xx^{j - \tau_j})}^2$ would appear at most $\tau_C^{(j)} - 1$ times (see the discussion in the proof of Lemma~\ref{lem:residual})
	\begin{align*}
	\sum_{t = 0}^T \eta_t \norm{{\xx}^{(t)} -  \xx^{(t - \tau_t)}}_2^2 \leq \frac{1}{4L} \sum_{t = 0}^{T } \tau_C \eta_t^2 \norm{\nabla f(\xx^{(t - \tau_t)})}^2 + \frac{\sigma^2}{4L} \sum_{t = 0}^{T} \eta_t^2 \,.
	\end{align*}
	Using again that  $\eta_t \leq \frac{1}{4 L \max\{ \tau_t, \tau_C \}}$ we get the statement of the lemma.
\end{proof}

\begin{proof}[Proof of Theorem~\ref{thm:homogeneous-adaptive}]
	We start by summing the descent Lemma~\ref{lem:descent} over the iterations $t=0,\dots, T$.
	\begin{align*}
	\sum_{t = 0}^T \eta_t \left(\frac{1}{2} \E \norm{\nabla f({\xx}^{(t)})}_2^2 +  \frac{1}{4} \E \norm{\nabla f({\xx}^{(t - \tau_t)})}^2\right)  &\leq \left(f({\xx}^{(0)}) - f^\star \right) + L \sigma^2 \sum_{t = 0}^T \eta_t^2 + \frac{L^2}{2} \sum_{t = 0}^T \eta_t \norm{{\xx}^{(t)} -  \xx^{(t - \tau_t)}}_2^2 \,.
	\end{align*}
	Next, we substitute Lemma~\ref{lem:residual_adaptive} into the last term,
	\begin{align*}
	\sum_{t = 0}^T \eta_t \left(\frac{1}{2} \E \norm{\nabla f({\xx}^{(t)})}_2^2 +  \frac{1}{4} \E \norm{\nabla f({\xx}^{(t - \tau_t)})}^2\right)  &\leq \left(f({\xx}^{(0)}) - f^\star \right) + L \sigma^2 \sum_{t = 0}^T \eta_t^2 \\
	& + \frac{1}{32} \sum_{t = 0}^{T } \eta_t \norm{\nabla f(\xx^{(t - \tau_t)})}^2 + \frac{\sigma^2 L }{8} \sum_{t = 0}^{T} \eta_t^2 \,.
	\end{align*}
	Rearranging we thus get 
	\begin{align*}
	\sum_{t = 0}^T \eta_t \E \norm{\nabla f({\xx}^{(t)})}_2^2  \leq  2 \left(f({\xx}^{(0)}) - f^\star \right) + 4 L \sigma^2 \sum_{t = 0}^T \eta_t^2  \,.
	\end{align*}
	We note that due to our choice of stepsizes \eqref{eq:adaptive_stepsizes}, $\eta_t \leq \eta$, it also holds that $\sum_{t = 0}^T \eta_t \geq \sum_{t : \tau_t \leq \tau_C} \eta \geq \frac{T + 1}{2} \eta$ since there are at least half of the iterations with the delay smaller than the average. 
	
	Using this, we estimate
	\begin{align*}
	\frac{1}{\sum_{t = 0}^T \eta_t}\sum_{t = 0}^T \eta_t \E \norm{\nabla f({\xx}^{(t)})}_2^2  \leq \frac{4}{(T + 1)\eta}\left(f({\xx}^{(0)}) - f^\star \right) + 8 L \sigma^2 \eta^2 \,.
	\end{align*}
	It remains to tune the stepsize $\eta$, i.e.\ to pick is such as to minimize the right hand side of this expression. See Lemma~17 in \cite{koloskova2020unified}.
\end{proof}

\subsection{Proof of  Theorem~\ref{thm:homogeneous}, convergence rate \eqref{eq:thm_bounded_delay}}
To prove the last claim of Theorem~\ref{thm:homogeneous} we take another approach and follow the perturbed iterate analysis \cite{mania2017:perturbed_iterate}.

We introduce a virtual sequence $\tilde \xx^t$ defined as
\begin{align*}
\tilde \xx^{(0)} = \xx^{(0)}, && \tilde \xx^{(t + 1)} = \tilde \xx^{(t)} - \eta \sum_{i \in \cA_t} \nabla F(\xx^{(t)}, \xi_{t + \hat{\tau}_t^i}),
\end{align*}
where we define $\cA_0 := \cC_0$, and $\hat \tau_t^i$ is the delay with which the corresponding gradient will be computed. That is, if we denote $j = t + \hat{\tau}_i^t$, then it will hold that $j - \tau_j = t$. This defines a virtual sequence and we do not have access to it during the execution of Algorithm~\ref{alg:async-homogeneous-general}. 

\begin{lemma}[Descent lemma]\label{lem:descent_second}
	Under Assumptions~\ref{a:stoch_noise} and\ref{a:lsmooth_nc}, if in Algorithm~\ref{alg:async-homogeneous-general} the stepsize $\eta_t < \frac{1}{2 L \tau_C }$ then it holds that 
	\begin{align*}
	\EE{t + 1}{f(\tilde{\xx}^{(t + 1)})} &\leq f(\tilde{\xx}^{(t)}) - \frac{\eta}{4} |\cA_t| \norm{\nabla f({\xx}^{(t)})}_2^2 + \frac{\eta}{2} |\cA_t| L^2 \norm{{\xx}^{(t)} - \tilde\xx^{(t)}}^2+  \frac{L \eta^2 \sigma^2 |\cA_t| }{2 } \,.
	\end{align*}
\end{lemma}
\begin{proof}
	Because function $f$ is $L$-smooth, we have 
	\begin{align*}
	\EE{t + 1}{f(\tilde{\xx}^{(t + 1)})} &= \EE{t + 1}{f\left(\tilde{\xx}^{(t)} -\eta \sum_{i \in \cA_t} \nabla F(\xx^{(t)}, \xi^{(t + \hat{\tau}_t^i)}) \right) }\\
	& \leq f(\tilde{\xx}^{(t)}) - \eta |\cA_t| \underbrace{\langle \nabla f(\tilde{\xx}^{(t)}), \nabla f(\xx^{(t)}) \rangle}_{=: T_1}  + \EE{t + 1}{\frac{L}{2} \eta^2 \underbrace{ \norm{  \sum_{i \in \cA_t} \nabla F(\xx^{(t)}, \xi_{t + \hat{\tau}_t^i}) }_2^2}_{=:T_2}} \,.
	\end{align*}
	We estimate the second term as
	\begin{align*}
	T_1 &= - \lin{ \nabla f({\xx}^{(t)}),  \nabla f(\tilde \xx^{(t)})} = - \frac{1}{2}\norm{\nabla f({\xx}^{(t)})}^2 - \frac{1}{2}\norm{\nabla f(\tilde{\xx}^{(t)})}^2 + \dfrac{1}{2} \norm{\nabla f({\xx}^{(t)}) - \nabla f(\tilde\xx^{(t )})}^2 \\
	&\leq - \frac{1}{2}\norm{\nabla f({\xx}^{(t)})}^2 + \dfrac{1}{2} \norm{\nabla f({\xx}^{(t)}) - \nabla f(\tilde\xx^{(t )})}^2  \,.
	\end{align*}
	For the last term, using the notation $~\pm a = a - a = 0~~\forall a$,
	\begin{align*}
	T_2 & =  \EE{t + 1}{\norm{ \sum_{i \in \cA_t} \nabla F(\xx^{(t)}, \xi_{t + \hat{\tau}_t^i}) \pm |\cA_t| \nabla f(\xx^{(t)}) }_2^2} \\
	& \stackrel{\eqref{eq:stochastic_noise}}{\leq} |\cA_t| \sigma^2 + |\cA_t|^2 \norm{\nabla f(\xx^{(t)}) }_2^2 \,.
	\end{align*}
	Combining this together, using $L$-smoothness to estimate $\norm{\nabla f({\xx}^{(t)}) - \nabla f(\tilde\xx^{(t)})}^2_2$ 
	we get
	\begin{align*}
	\EE{t + 1}{f(\tilde{\xx}^{(t + 1)})} &\leq f(\tilde{\xx}^{(t)}) - \left(\frac{\eta}{2} |\cA_t| - \frac{\eta^2 L |\cA_t|^2}{2} \right)\norm{\nabla f({\xx}^{(t)})}_2^2 + \frac{\eta}{2} |\cA_t| L^2 \norm{{\xx}^{(t)} - \tilde\xx^{(t)}}^2+  \frac{L \eta^2 \sigma^2 |\cA_t| }{2 } \,.
	\end{align*}
	Using that $\eta \leq \frac{1}{2 L \tau_C} \leq \frac{1}{2 L |\cA_t|}$ we get statement of the Lemma. 
\end{proof}

\begin{lemma}[Estimation of the residual]\label{lem:residual_bounded_grad}
	Under Assumptions~\ref{a:stoch_noise}, \ref{a:lsmooth_nc}, iterated of Algorithm~\ref{alg:async-homogeneous-general} with the constant stepsize  $\eta_t \equiv \eta$ with $\eta \leq \frac{1}{2L \tau_C}$ satisfy
	\begin{align*}
	\E \norm{{\xx}^{(t)} -  \tilde \xx^{(t)}}_2^2 \leq \tau_C^2 \eta^2 G^2 + \eta^2 \tau_C \sigma^2 \,.
	\end{align*}
\end{lemma}
\begin{proof}
	\begin{align*}
	\E \norm{{\xx}^{(t)} -  \tilde \xx^{(t)}}_2^2 &= \E \norm{ \sum_{j \in \cC_t} \eta \nabla F(\xx^{(j)}, \xi_{j + \hat\tau_j}) }_2^2 \stackrel{\eqref{eq:stochastic_noise}}{\leq }\E \norm{ \sum_{j \in \cC_t} \eta \nabla f(\xx^{(j)}) }_2^2 + \eta^2 \tau_C^{(t)} \sigma^2 \\
	&\stackrel{\eqref{eq:norm_of_sum}}{\leq}  \tau_C^{(t)} \sum_{j \in \cC_t} \eta^2 \E \norm{ \nabla f(\xx^{(j)}) }_2^2 + \eta^2 \tau_C^{(t)} \sigma^2 \\
	& \stackrel{\eqref{eq:bounded_gradient}}{\leq } (\tau_C^{(t)})^2 \eta^2 G^2 + \eta^2 \tau_C^{(t)} \sigma^2 \,.
	\end{align*}
\end{proof}

We are now ready to prove the second claim of  Theorem~\ref{thm:homogeneous}. 
\begin{proof}[Proof of Theorem~\ref{thm:homogeneous}, convergence rate \eqref{eq:thm_bounded_delay}]
	We start by summing over $t=0,\dots,T$ the descent Lemma~\ref{lem:descent_second}. We also divide it by $\eta$,
	\begin{align*}
	\sum_{t = 0}^T \frac{1}{4} |\cA_t| \norm{\nabla f({\xx}^{(t)})}_2^2  &\leq \frac{1}{\eta}\left(f({\xx}^{(0)}) - f^\star \right) + \frac{L \eta \sigma^2}{2 }\sum_{t = 0}^T |\cA_t| + \frac{L^2}{2} \sum_{t = 0}^T |\cA_t| \norm{{\xx}^{(t)} - \tilde\xx^{(t)}}^2\,.
	\end{align*}
	We further use Lemma~\ref{lem:residual_bounded_grad} for the last term 
	\begin{align*}
	\sum_{t = 0}^T \frac{1}{4} |\cA_t| \norm{\nabla f({\xx}^{(t)})}_2^2  &\leq \frac{1}{\eta}\left(f({\xx}^{(0)}) - f^\star \right) + \frac{L \eta \sigma^2}{2 }\sum_{t = 0}^T |\cA_t| + \frac{L^2}{2}  \left(\tau_C^2 \eta^2 G^2 + \eta^2 \tau_C \sigma^2  \right)\sum_{t = 0}^T |\cA_t|  \,.
	\end{align*}
	We further use that $\eta \leq \frac{1}{2 L \tau_C}$ for the last $\sigma$ term and divide the full inequality by $\frac{1}{4} \cW_T$, where we defined $\cW_T = \sum_{t = 0}^T |\cA_t| $
	\begin{align*}
	\frac{1}{\cW_T}\sum_{t = 0}^T |\cA_t| \norm{\nabla f({\xx}^{(t)})}_2^2  &\leq \frac{4}{\eta \cW_T} \left(f({\xx}^{(0)}) - f^\star \right) + 4 L \eta \sigma^2 + 2 L^2 \tau_C^2 \eta^2 G^2 \,.
	\end{align*}
	Note that because at every step $t$ only one of the gradients is getting applied, $T \leq \sum_{t = 0}^T |\cA_t| \leq  T + \tau_C \leq 2 T$ for $T \geq \tau_C$. 
	
	It is left to tune the stepsize using Lemma~17 in \cite{koloskova2020unified} to get the final convergence rate. 
\end{proof}

\subsection{Proof of the Theorem~\ref{thm:heterogeneous}}
We first re-state the theorem
\themheterogeneous*
We utilize again the perturbed iterate technique \cite{mania2017:perturbed_iterate}. We introduce a virtual sequence $\tilde\xx^{(t)}$ as
\begin{align*}
\tilde \xx^{(0)} = \xx^{(0)} && \tilde \xx^{(t + 1)} = \tilde \xx^{(t)} - \eta \nabla F_{k_t}(\xx^{(t)}, \xi_{t + \hat \tau_t}),
\end{align*}
where we define $\hat\tau_t$ as the delay with which the corresponding gradient will be computed. If we denote $j = t + \hat\tau_t$, then it holds that $j - \tau_j = t$. 
\begin{lemma}[Descent Lemma] \label{lem:descent_het}
	Under Assumptions \ref{a:stoch_noise}, \ref{a:heterogeneity}, \ref{a:lsmooth_nc}, for Algorithm~\ref{alg:async-general}  with the stepsize $\eta_t \leq \frac{1}{4 L}$ it holds that 
	\begin{align}
	\EE{t + 1}{f(\tilde{\xx}^{(t + 1)})} &\leq f(\tilde{\xx}^{(t)}) - \frac{\eta}{4} \norm{\nabla f({\xx}^{(t)})}_2^2 + \frac{L \eta^2 \sigma^2}{2} +  L \eta^2  \zeta^2 + \frac{\eta L^2 }{2}\norm{{\xx}^{(t)} -  \tilde\xx^{(t)}}_2^2  \,.
	\end{align}
\end{lemma}
\begin{proof}
	Because the function $f$ is $L$-smooth, we have 
	\begin{align*}
	\EE{t + 1}{f(\tilde{\xx}^{(t + 1)})} &= \EE{t + 1}{f\left(\tilde{\xx}^{(t)} - \eta \nabla F_{k_t}(\xx^{(t)}, \xi_{t + \hat{\tau}_t}) \right) }\\
	& \leq f(\tilde{\xx}^{(t)}) - \eta \underbrace{\langle \nabla f(\tilde{\xx}^{(t)}), \nabla f(\xx^{(t)}) \rangle}_{=: T_1}  + \EE{t + 1}{\frac{L}{2} \eta^2 \underbrace{ \norm{  \nabla F_{k_t}(\xx^{(t)}, \xi_{t + \hat{\tau}_t}) }_2^2}_{=:T_2}} \,,
	\end{align*}
	where expectation is taken over both the stochastic noise $\xi$ and sampled index $j_t$.
	We estimate terms $T_1$ and $T_2$ separately
	\begin{align*}
	T_1 &= - \frac{\eta}{2}\norm{\nabla f({\xx}^{(t)})}^2 - \frac{\eta}{2}\norm{\nabla f(\tilde{\xx}^{(t)})}^2  + \dfrac{\eta}{2} \norm{\nabla f({\xx}^{(t)}) - \nabla f(\tilde\xx^{(t)})}^2 \\
	&\leq - \frac{\eta}{2}\norm{\nabla f({\xx}^{(t)})}^2 + \dfrac{\eta}{2} \norm{\nabla f({\xx}^{(t)}) - \nabla f(\tilde\xx^{(t)})}^2  \,.
	\end{align*}
	For the last term, using the notation $~\pm a = a - a = 0~~\forall a$,
	\begin{align*}
	T_2 & =  \EE{t + 1}{\norm{ \nabla F_{k_t}(\xx^{(t)}, \xi_{t + \hat{\tau}_t})  \pm \nabla f_{j_t}(\xx^{(t)}) \pm \nabla f(\xx^{(t)})}_2^2} \\
	& \stackrel{\eqref{eq:stochastic_noise}}{\leq} \sigma^2 + 2 \EE{k_t}{\norm{\nabla f_{k_t}(\xx^{(t)}) - \nabla f(\xx^{(t)})}_2^2} + 2 \norm{\nabla f(\xx^{(t)}) }_2^2 \\
	& \stackrel{\eqref{eq:bound_heterogeniety}}{\leq} \sigma^2  + 2 \zeta^2 + 2 \norm{\nabla f(\xx^{(t)})}_2^2 \,.
	\end{align*}
	Combining this together and using $L$-smoothness to estimate $\norm{\nabla f({\xx}^{(t)}) - \nabla f(\tilde\xx^{(t)})}^2_2$ we get
	\begin{align*}
	\EE{t + 1}{f(\tilde{\xx}^{(t + 1)})} &\leq f(\tilde{\xx}^{(t)}) - \left(\frac{\eta}{2} - L \eta^2\right)\norm{\nabla f({\xx}^{(t)})}_2^2 + \frac{\eta}{2} L^2 \norm{{\xx}^{(t)} - \tilde\xx^{(t)}}^2+  \frac{L \eta^2 \sigma^2}{2} + L \eta^2  \zeta^2 \,.
	\end{align*}
	Applying $\eta \leq  \frac{1}{4 L}$ we get statement of the lemma.
\end{proof}
\subsubsection{Proof of Theorem~\ref{thm:heterogeneous}, convergence rate \eqref{eq:thm_het_first}}
\begin{lemma}[Estimation of the distance $\norm{{\xx}^{(t)} -  \tilde\xx^{(t)}}_2^2$] \label{lem:distance}
	Under Assumptions \ref{a:stoch_noise}, \ref{a:heterogeneity}, \ref{a:lsmooth_nc}, for Algorithm~\ref{alg:async-general}  with the stepsize $\eta_t \leq \frac{1}{4 L \sqrt{\tau_{C}\tau_{\max}}}$ it holds that 
	\begin{align*}
	\frac{1}{T + 1} \sum_{t = 0}^T \E \norm{{\xx}^{(t)} -  \tilde\xx^{(t)}}_2^2 \leq \frac{\eta \sigma^2 }{4 L } + \frac{2 \eta^2 \tau_C }{T + 1} \frac{1}{n}\sum_{j = 1}^n \zeta_{j}^2 \bar\tau_j + \frac{1}{8L^2(T + 1)}\sum_{t = 0}^T \E \norm{\nabla f(\xx^{(t)}) }_2^2 \,.
	\end{align*}
\end{lemma}
\begin{proof}
	\begin{align*}
	\E \norm{{\xx}^{(t)} -  \tilde\xx^{(t)}}_2^2 &= \E \eta^2 \norm{ \sum_{i \in \cC_t} \nabla F_{j_i}(\xx^{(i)}, \xi_{i + \hat\tau_i})}_2^2 \stackrel{\eqref{eq:stochastic_noise}}{\leq } \eta^2 \tau_C \sigma^2 + \eta^2 \E \norm{\sum_{i \in \cC_t} \nabla f_{j_i}(\xx^{(i)})}_2^2 \\
	&\stackrel{\eqref{eq:norm_of_sum}}{\leq}  \eta^2 \tau_C \sigma^2 + 2 \eta^2 \E \norm{\sum_{i \in \cC_t} \nabla f_{j_i}(\xx^{(i)}) - \nabla f(\xx^{(i)})}_2^2  + 2 \eta^2 \E \norm{\sum_{i \in \cC_t} \nabla f(\xx^{(i)}) }_2^2 \\
	& \stackrel{\eqref{eq:norm_of_sum}}{\leq} \eta^2 \tau_C \sigma^2 + 2 \eta^2 \tau_C^{(t)} \E \sum_{i \in \cC_t} \zeta_{j_i}^2  + 2 \eta^2 \tau_C \E \sum_{i \in \cC_t} \norm{\nabla f(\xx^{(i)}) }_2^2 \,.
	\end{align*}
	Averaging over $T$, we get
	\begin{align*}
	\frac{1}{T + 1} \sum_{t = 0}^T \E \norm{{\xx}^{(t)} -  \tilde\xx^{(t)}}_2^2 \leq \eta^2 \tau_C \sigma^2 + 2 \eta^2 \tau_C \frac{1}{T + 1}\sum_{t = 0}^T \E \sum_{i \in \cC_t} \zeta_{j_i}^2 + 2 \eta^2 \tau_C \frac{1}{T + 1} \sum_{t = 0}^T \E \sum_{i \in \cC_t} \norm{\nabla f(\xx^{(i)}) }_2^2 \,.
	\end{align*}
	We note that in the second term each of $\zeta_j$ appears exactly $\tau_j^{sum}$ times, where $\tau_j^{sum}$ is the sum of the all the delays that happened on the node $j$. In the last term, we estimate the number of appearance of each of $\norm{\nabla f(\xx^{(i)}) }_2^2$ bt $\tau_{\max}$, thus
	\begin{align*}
	\frac{1}{T + 1} \sum_{t = 0}^T \E \norm{{\xx}^{(t)} -  \tilde\xx^{(t)}}_2^2 \leq \eta^2 \tau_C \sigma^2 + 2 \eta^2 \tau_C \E \frac{1}{T + 1}\sum_{j = 1}^n \zeta_{j}^2 \tau_j^{sum} + 2 \eta^2 \tau_C \tau_{\max} \frac{1}{T + 1}\sum_{t = 0}^T \E \norm{\nabla f(\xx^{(t)}) }_2^2 \,,
	\end{align*}
	we further use that number of times $T_j$ that every node $j$ got sampled are equal in expectation because of uniform sampling in line 6 of Algorithm~\ref{alg:async-general}. Thus, 
	\begin{align*}
	\frac{1}{T + 1} \sum_{t = 0}^T \E \norm{{\xx}^{(t)} -  \tilde\xx^{(t)}}_2^2 \leq \eta^2 \tau_C \sigma^2 + 2 \eta^2 \tau_C \frac{1}{T + 1} \frac{1}{n}\sum_{j = 1}^n \zeta_{j}^2 \bar\tau_j + 2 \eta^2 \tau_C \tau_{\max} \frac{1}{T + 1}\sum_{t = 0}^T \E \norm{\nabla f(\xx^{(t)}) }_2^2 \,.
	\end{align*}
	Using that $\eta \leq \frac{1}{4 L \sqrt{\tau_{C}\tau_{\max}}}$ we get the statement of the lemma.
\end{proof}

\begin{proof}[Proof of Theorem~\ref{thm:heterogeneous}, \eqref{eq:thm_het_first}]
	First, averaging the descent Lemma~\ref{lem:descent_second},
	\begin{align*}
	\frac{1}{T + 1} \sum_{t = 0}^T \E\norm{\nabla f({\xx}^{(t)})}_2^2 \leq \frac{4}{\eta (T + 1)} \left(f( \xx^{0}) -  f(\xx^{{T}})\right) + 2 L \eta \sigma^2 + 4 L \eta \zeta^2+ \frac{2 L^2}{T + 1} \sum_{t = 0}^T \E \norm{{\xx}^{(t)} -  \tilde\xx^{(t)}}_2^2 \,.
	\end{align*}
	Now plugging in the result of Lemma~\ref{lem:distance}, we get
	\begin{align*}
	\frac{1}{T + 1} \sum_{t = 0}^T \E\norm{\nabla f({\xx}^{(t)})}_2^2 &\leq \frac{4}{\eta (T + 1)} \left(f( \xx^{0}) -  f(\xx^{{T}})\right) + 2 L \eta \sigma^2 +  4 L \eta \zeta^2 + \frac{L  \eta \sigma^2 }{2} \\
	& + \frac{4 L^2 \eta^2 \tau_C }{T + 1} \frac{1}{n}\sum_{j = 1}^n \zeta_{j}^2 \bar\tau_j + \frac{1}{4(T + 1)}\sum_{t = 0}^T \E \norm{\nabla f(\xx^{(t)}) }_2^2 \,.
	\end{align*}
	Rearranging terms we thus get
	\begin{align*}
	\frac{1}{2(T + 1)} \sum_{t = 0}^T \E\norm{\nabla f({\xx}^{(t)})}_2^2 &\leq \frac{4}{\eta (T + 1)} \left(f( \xx^{0}) -  f(\xx^{{T}})\right) + 3 L \eta \sigma^2 +  4 L \eta \zeta^2 + \frac{4 L^2 \eta^2 \tau_C }{T + 1} \frac{1}{n}\sum_{j = 1}^n \zeta_{j}^2 \bar\tau_j 
	\end{align*}
	It is only left to tune the stepsize using Lemma~17 in \cite{koloskova2020unified}.
\end{proof}

\subsubsection{Proof of Theorem~\ref{thm:heterogeneous}, convergence rate \eqref{eq:het_bounded}.}
\begin{lemma}[Estimation of the distance $\norm{{\xx}^{(t)} -  \tilde\xx^{(t)}}_2^2$] \label{lem:distance2}
	Under Assumptions \ref{a:stoch_noise}, \ref{a:heterogeneity}, \ref{a:lsmooth_nc}, \ref{a:bounded_gradient} for Algorithm~\ref{alg:async-general}  with the stepsize $\eta_t \equiv \eta \leq \frac{1}{4 L \tau_{C}}$ it holds that 
	\begin{align*}
	\frac{1}{T + 1} \sum_{t = 0}^T \E \norm{{\xx}^{(t)} -  \tilde\xx^{(t)}}_2^2 \leq  \frac{\eta \sigma^2 }{4L }+ \eta^2 \tau_C^2G^2 \,.
	\end{align*}
\end{lemma}
\begin{proof}
	We start our proof similar way as before
	\begin{align*}
	\E \norm{{\xx}^{(t)} -  \tilde\xx^{(t)}}_2^2 &= \E \eta^2 \norm{ \sum_{i \in \cC_t} \nabla F_{j_i}(\xx^{(i)}, \xi_{i + \hat\tau_i})}_2^2 \stackrel{\eqref{eq:stochastic_noise}}{\leq } \eta^2 \tau_C \sigma^2 + \eta^2 \E \norm{\sum_{i \in \cC_t} \nabla f_{j_i}(\xx^{(i)})}_2^2 \\
	&\stackrel{\eqref{eq:norm_of_sum}}{\leq}  \eta^2 \tau_C \sigma^2 + \eta^2 \tau_C \sum_{i \in \cC_t} \E \norm{ \nabla f_{j_i}(\xx^{(i)})}_2^2\\
	& \stackrel{\eqref{eq:bounded_gradient}}{\leq} \eta^2 \tau_C \sigma^2 + \eta^2 \tau_C^2G^2 \\
	& ~{\leq} ~\frac{\eta \sigma^2 }{4L }+ \eta^2 \tau_C^2G^2 
	\end{align*}
	where on the last line we used that stepsize $\eta \leq \frac{1}{4 L \tau_{C}}$.
\end{proof}
\begin{proof}[Proof of the Theorem~\ref{thm:heterogeneous}, \eqref{eq:het_bounded}]
	We start by averaging the descent Lemma~\ref{lem:descent_second},
	\begin{align*}
	\frac{1}{T + 1} \sum_{t = 0}^T \E\norm{\nabla f({\xx}^{(t)})}_2^2 \leq \frac{4}{\eta (T + 1)} \left(f( \xx^{0}) -  f(\xx^{{T}})\right) + 2 L \eta \sigma^2 + 4 L \eta \zeta^2+ \frac{2 L^2}{T + 1} \sum_{t = 0}^T \E \norm{{\xx}^{(t)} -  \tilde\xx^{(t)}}_2^2 \,.
	\end{align*}
	We now plug in the results of Lemma~\ref{lem:distance2} and get
	\begin{align*}
	\frac{1}{T + 1} \sum_{t = 0}^T \E\norm{\nabla f({\xx}^{(t)})}_2^2 \leq \frac{4}{\eta (T + 1)} \left(f( \xx^{0}) -  f(\xx^{{T}})\right) + 3 L \eta \sigma^2 + 4 L \eta \zeta^2+ 2 L^2 \eta^2 \tau_C^2G^2 \,.
	\end{align*}
	It is only left to tune the stepsize using Lemma~17 in \cite{koloskova2020unified}.
\end{proof}

\subsection{Proof of Lemma~\ref{lem:speedup}}
In this section we prove Lemma~\ref{lem:speedup} that estimates expected execution time of Algorithm~\ref{alg:async-general} with concurrency $\tau_C = C$, and the expected time of mini-batch SGD with the same concurrency i.e.\ batch size equal to $\tau_C = C$. 
\begin{proof}
We start by proving the first claim. 	
\paragraph{Time of Asynchronous Algorithm~\ref{alg:async-general}.}	
Assume the concurrency is $C = 1$. Then, Algorithm~\ref{alg:async-general} is synchronous, and as we sample every client with equal probability (line 6 of Algorithm~\ref{alg:async-general}), the expected time to compute one gradient is equal to $\frac{1}{n} \sum_{i = 1}^n \Delta_i$. 

To calculate the estimated time with concurrency $C > 1$ we can view the Algorithm~\ref{alg:async-general} as having $C$ independent copies of the previous process run in parallel. Thus, in the same time $\frac{1}{n} \sum_{i = 1}^n \Delta_i$ in expectation Algorithm~\ref{alg:async-general} will compute $C$ gradients. 

\paragraph{Time of Mini-batch SGD.}
The expected time of mini-batch SGD of size $C$ is equal to $\E \max\{\Delta_{i_1}, \dots, \Delta_{i_{C}}\}$ with each $i_j \sim \operatorname{Uniform}[1, n]$. Denote a random variable $X = \max\{\Delta_{i_1}, \dots, \Delta_{i_{C}}\}$ that takes values within $\Delta_1, \dots \Delta_n$. 
Since $i_j$ are independent from each other, 
\begin{align*}
\PP{}{X \leq \Delta_k}  = \prod_{j = 1}^C \PP{}{\Delta_{i_j} \leq \Delta_k} = \prod_{j = 1}^C \PP{}{X \leq \Delta_k}  = \prod_{j = 1}^C \PP{}{{i_j} \leq k} = \left(\frac{k}{n}\right)^C \,.
\end{align*}
Thus, 
\begin{align*}
\PP{}{X = \Delta_k} =  \frac{k^C - (k - 1)^C}{n^C} \,.
\end{align*}
And therefore,
\begin{align*}
\E X &= \sum_{k = 1}^n\PP{}{X = \Delta_k} \Delta_k \,. \qedhere
\end{align*}
\end{proof}

\section{Experiments}

In this set of experiments we aim to illustrate the dependence %
on the maximum delay $\tau_{\max}$ in Theorem~\ref{thm:homogeneous}, as depicted in Equation \eqref{eq:thm_first_part}. For this, we set the stochastic noise $\sigma$ to zero. In this case Theorem~\ref{thm:homogeneous}, Equation \eqref{eq:thm_first_part} predicts that to reach an $\epsilon$ accuracy, Algorithm~\ref{alg:async-homogeneous-general} needs $T = \cO\left(\frac{\sqrt{\tau_{\max} \tau_C}}{\epsilon}\right)$ iterations. In our experiments we fix $\tau_C = 2$, $\epsilon = 10^{-14}$. Since $\tau_C = 2$, we have two workers. We vary the relative speed of the second worker, and thus affecting the maximum delay:\ if the second worker is $x$ times slower than the first worker, then the maximum delay $\tau_{\max} = x$. We measure the time $T$ to reach the accuracy $\epsilon$. Since all the other parameters are constant, it holds that $T = C_1 \sqrt{\tau_{\max}}$. 

We perform experiments on two different functions:  
\begin{itemize}
	\item[(i)]quadratic function $f(\xx) = \frac{1}{2} \norm{A\xx - \bb}_2^2$, $\xx, \bb \in \R^{10}$, $b_i \sim \cN(0, 1),i \in [1, 10]$, $A \in \R^{10 \times 10}$ is a random matrix with $\lambda_{\max}(A) = 2$, $\lambda_{\min}(A) = 1$ and the rest of eigenvalues are equally spaced in between. 
	\item[(ii)] logistic regression function $f(\xx) = \frac{1}{m} \sum_{j = 1}^m \log(1 + \exp(-b_j \aa_j^\top \xx))$, where each $b_j$ is sampled uniformly at random from the set $\{-1, 1\}$, and $\aa_j \sim \cN \left(0, 1\right)^{20}$, $\xx \in \R^{20}$, $m = 100$. 
\end{itemize}

We estimate the error as the average over the last 30 iterations $\hat{\epsilon}_T = \frac{1}{30} \sum_{i = 0}^{29} \norm{\nabla f(\xx_{T - i})}_2$. We tune the stepsize $\eta$ for every experiment \emph{separately} over the logarithmic grid between $10^{-5}$ and $10^2$ ensuring that the optimal stepsize value is not on the edge of the grid. 

Figure~\ref{fig:guadratic} shows the resulting dependence of $T$ on $\tau_{\max}$ for  the quadratic function (i), and Figure~\ref{fig:logistic} for the logistic regression function (ii). In both cases we see that $T$ has linear dependence on $\sqrt{\tau_{\max}}$ confirming our theory. 

\begin{figure}
	\centering
	\begin{minipage}{0.8\linewidth}
		\includegraphics[width=0.48\linewidth]{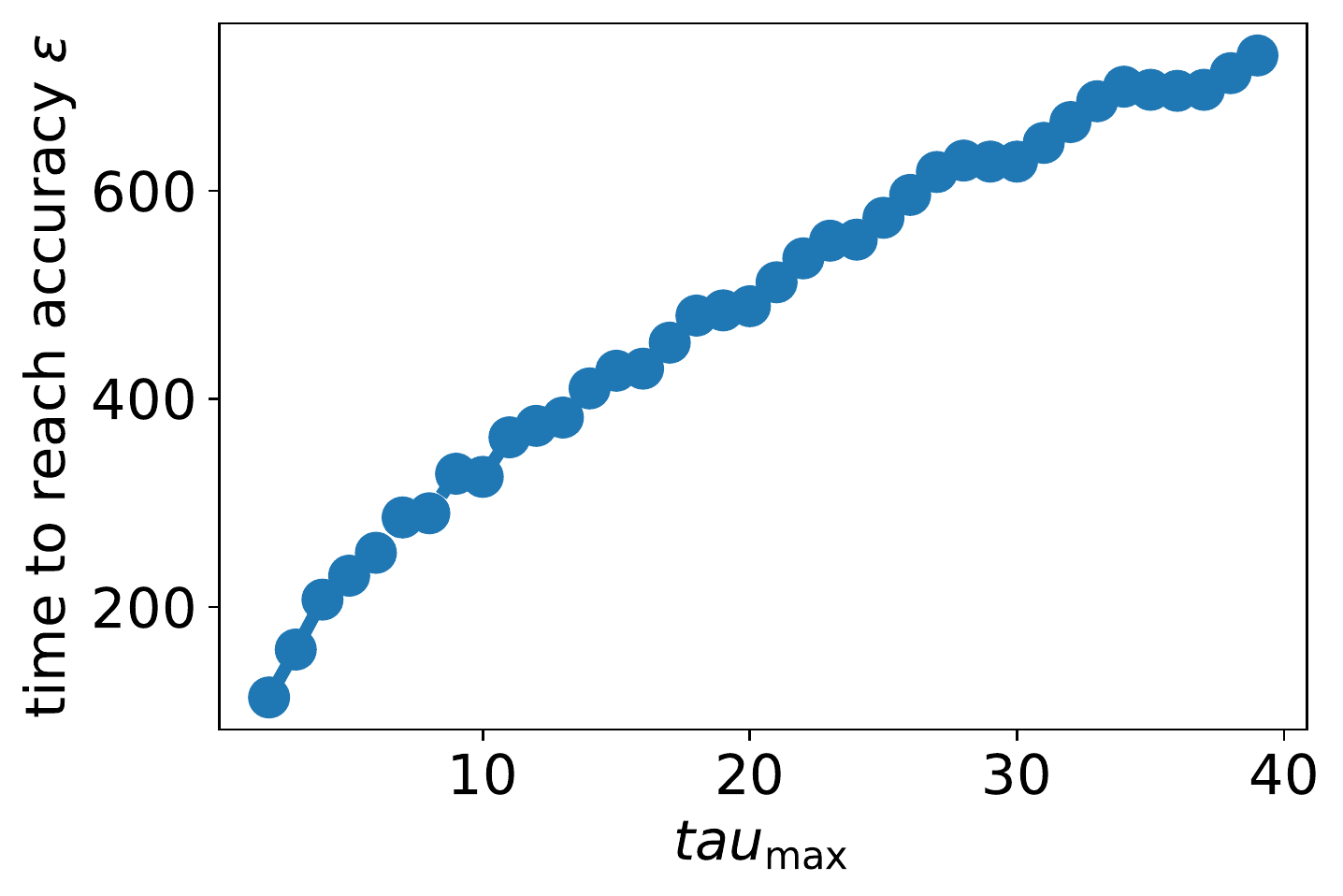}
		\includegraphics[width=0.48\linewidth]{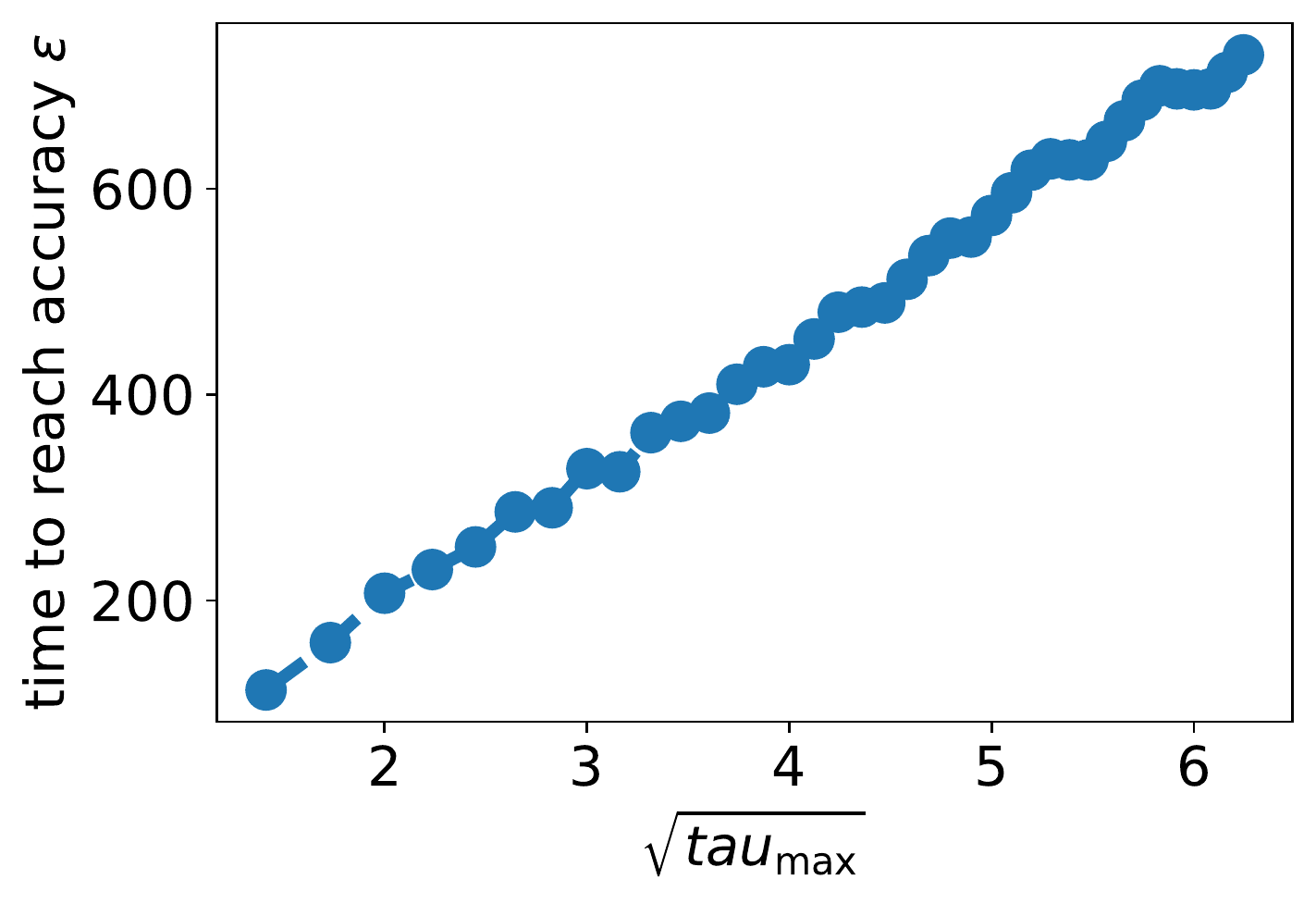}
		\vspace{-2mm}
		\caption{Verification of $\sqrt{\tau_{\max}}$ dependence on random quadratic function (i). We see that $T$ has linear dependence on $\sqrt{\tau_{\max}}$}\label{fig:guadratic}
	\end{minipage}
	\hfill
	\centering
	\begin{minipage}{0.8\linewidth}
		\includegraphics[width=0.48\linewidth]{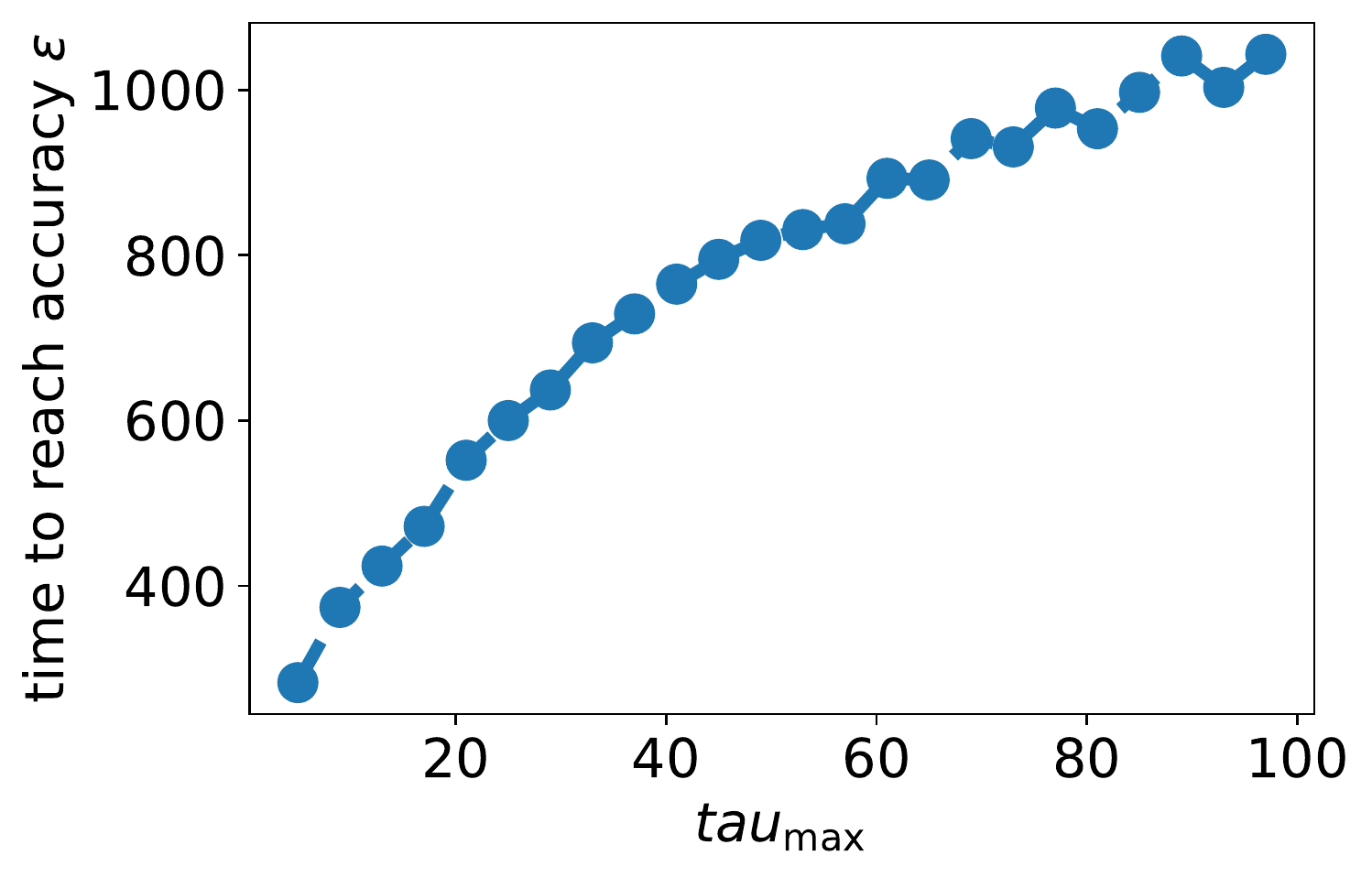}
		\includegraphics[width=0.48\linewidth]{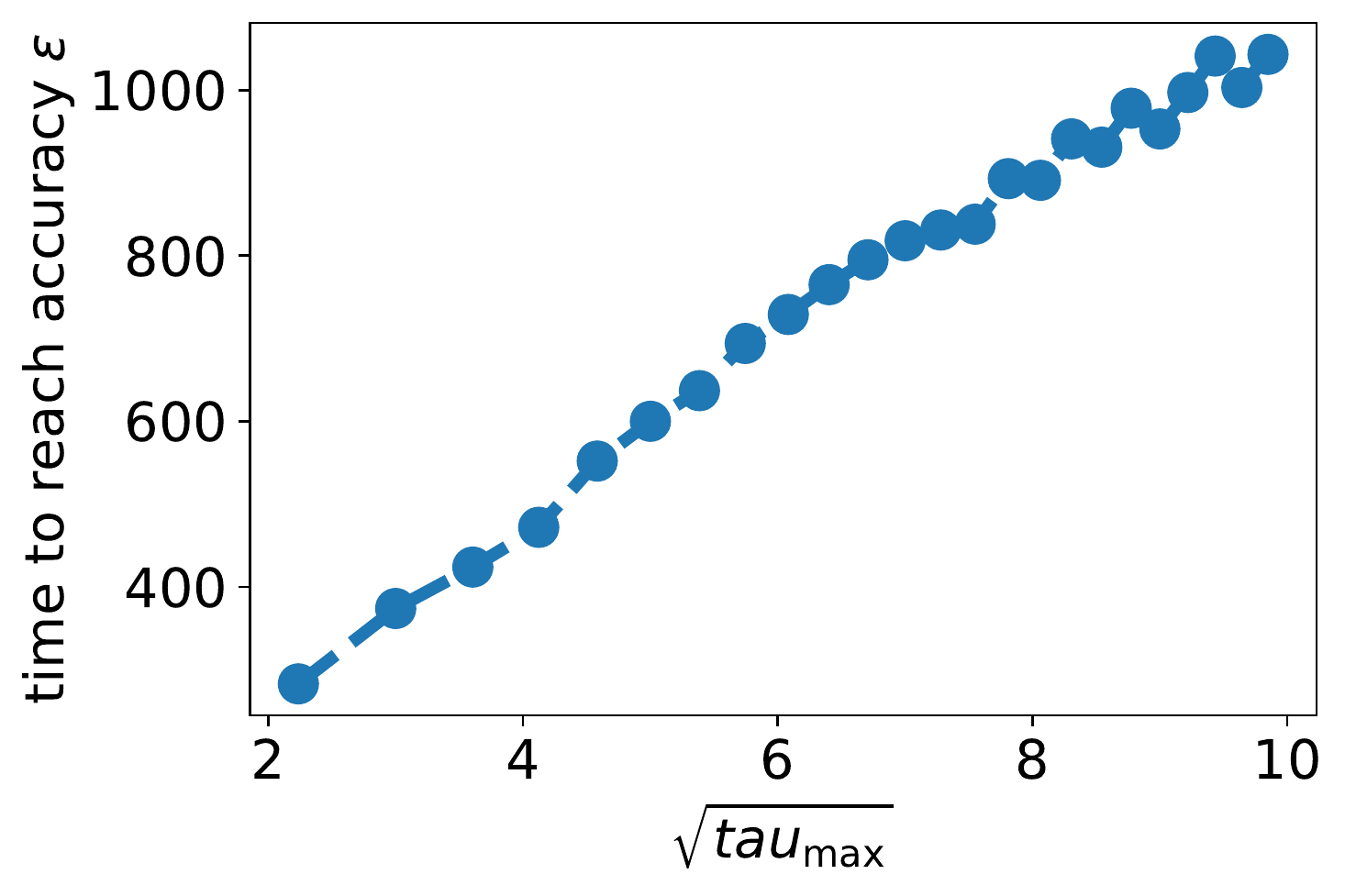}
		\vspace{-2mm}
		\caption{Verification of $\sqrt{\tau_{\max}}$ dependence on random logistic regression function (ii). We see that $T$ has linear dependence on $\sqrt{\tau_{\max}}$}\label{fig:logistic}
	\end{minipage}
	
\end{figure}

\end{document}